\documentclass[conference]{IEEEtran}
\IEEEoverridecommandlockouts
\usepackage{cite}
\usepackage{amsmath,amssymb,amsfonts}
\usepackage{graphicx}
\def\BibTeX{{\rm B\kern-.05em{\sc i\kern-.025em b}\kern-.08em
    T\kern-.1667em\lower.7ex\hbox{E}\kern-.125emX}}

\usepackage{subfigure}
\usepackage{epstopdf}
\usepackage{epsfig}
\usepackage{comment}
\usepackage[ruled,linesnumbered,vlined]{algorithm2e}
\SetKw{Continue}{continue}
\usepackage{amsthm}
\usepackage{accents}
\usepackage{booktabs}
\usepackage{tabularx}
\usepackage{multirow}

\usepackage{url}

\newtheorem{theorem}{Theorem}
\newtheorem{lemma}{Lemma}
\newtheorem{definition}{Definition}

\newcommand{\ie}{{i.e.}}
\newcommand{\eg}{{e.g.}}
\newcommand{\et}{{et al.}}

\newcommand{\resp}{{resp.}}

\newcommand{\oie}{i.e.}

\begin{document}

\title{Secure Federated Submodel Learning
\thanks{This work was supported in part by Science and Technology Innovation 2030 -- ``New Generation Artificial Intelligence" Major Project No. 2018AAA0100905, in part by China NSF grant 61972252, 61972254, 61672348, and 61672353, in part by the Open Project Program of the State Key Laboratory of Mathematical Engineering and Advanced Computing 2018A09, and in part by Alibaba Group through Alibaba Innovation Research (AIR) Program. The opinions, findings, conclusions, and recommendations expressed in this paper are those of the authors and do not necessarily reflect the views of the funding agencies or the government.}
\thanks{F. Wu is the corresponding author.}
}

\author{
\IEEEauthorblockN{Chaoyue~Niu$^1$, Fan~Wu$^1$, Shaojie~Tang$^2$, Lifeng Hua$^3$, Rongfei Jia$^3$, Chengfei Lv$^3$, Zhihua Wu$^3$, and Guihai~Chen$^1$}
\IEEEauthorblockA{$^1$Shanghai Key Laboratory of Scalable Computing and Systems, Shanghai~Jiao~Tong~University, China\\
$^2$Naveen Jindal School of Management, University of Texas at Dallas, USA\\
$^3$Alibaba Group, Hangzhou and Beijing, China\\
Email: $^1$\{rvince, wu-fan\}@sjtu.edu.cn; gchen@cs.sjtu.edu.cn; $^2$shaojie.tang@utdallas.edu;\\$^3$\{issac.hlf, rongfei.jrf, chengfei.lcf, zhihua.wzh\}@alibaba-inc.com
}
}

\maketitle

\begin{abstract}
Federated learning was proposed with an intriguing vision of achieving collaborative machine learning among numerous clients without uploading their private data to a cloud server. However, the conventional framework requires each client to leverage the full model for learning, which can be prohibitively inefficient for resource-constrained clients and large-scale deep learning tasks. We thus propose a new framework, called federated submodel learning, where clients download only the needed parts of the full model, namely submodels, and then upload the submodel updates. Nevertheless, the ``position" of a client's truly required submodel corresponds to her private data, and its disclosure to the cloud server during interactions inevitably breaks the tenet of federated learning. To integrate efficiency and privacy, we have designed a secure federated submodel learning scheme coupled with a private set union protocol as a cornerstone. Our secure scheme features the properties of randomized response, secure aggregation, and Bloom filter, and endows each client with a customized plausible deniability, in terms of local differential privacy, against the position of her desired submodel, thus protecting her private data. We further instantiated our scheme with the e-commerce recommendation scenario in Alibaba, implemented a prototype system, and extensively evaluated its performance over 30-day Taobao user data. The analysis and evaluation results demonstrate the feasibility and scalability of our scheme from model accuracy and convergency, practical communication, computation, and storage overheads, as well as manifest its remarkable advantages over the conventional federated learning framework.
\end{abstract}

\begin{IEEEkeywords}
federated submodel learning, private set union, randomized response, local differential privacy, secure aggregation, Bloom filter, e-commerce recommendation
\end{IEEEkeywords}

\section{Introduction}


\subsection{Motivating Industrial Scenario in Alibaba}

\begin{figure*}[t]
\centering
\subfigure[Conventional Federated Learning]{\label{fig:system:model:full}
\includegraphics[width=0.95\columnwidth]{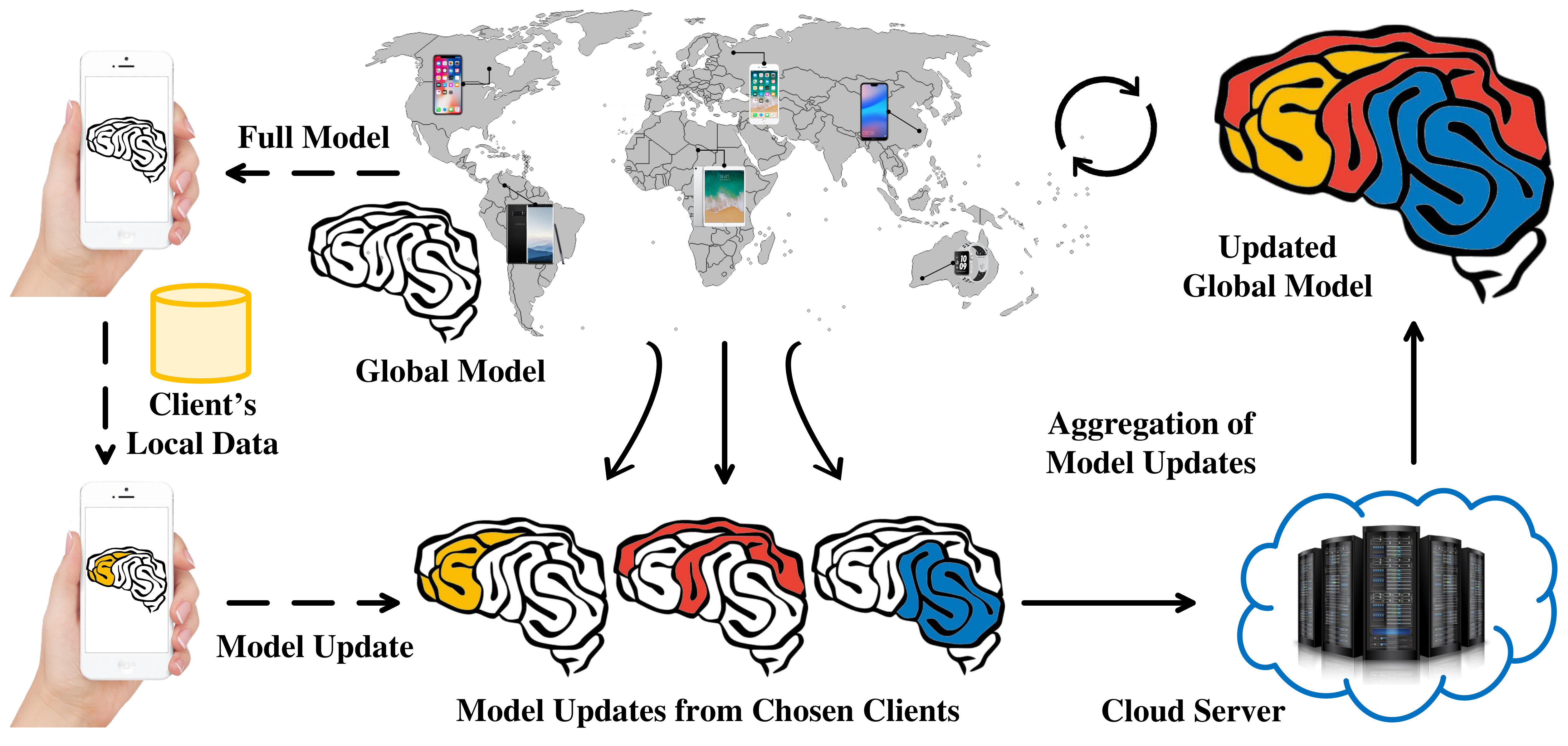}}
\subfigure[Federated Submodel Learning]{\label{fig:system:model:sub}
\includegraphics[width=0.95\columnwidth]{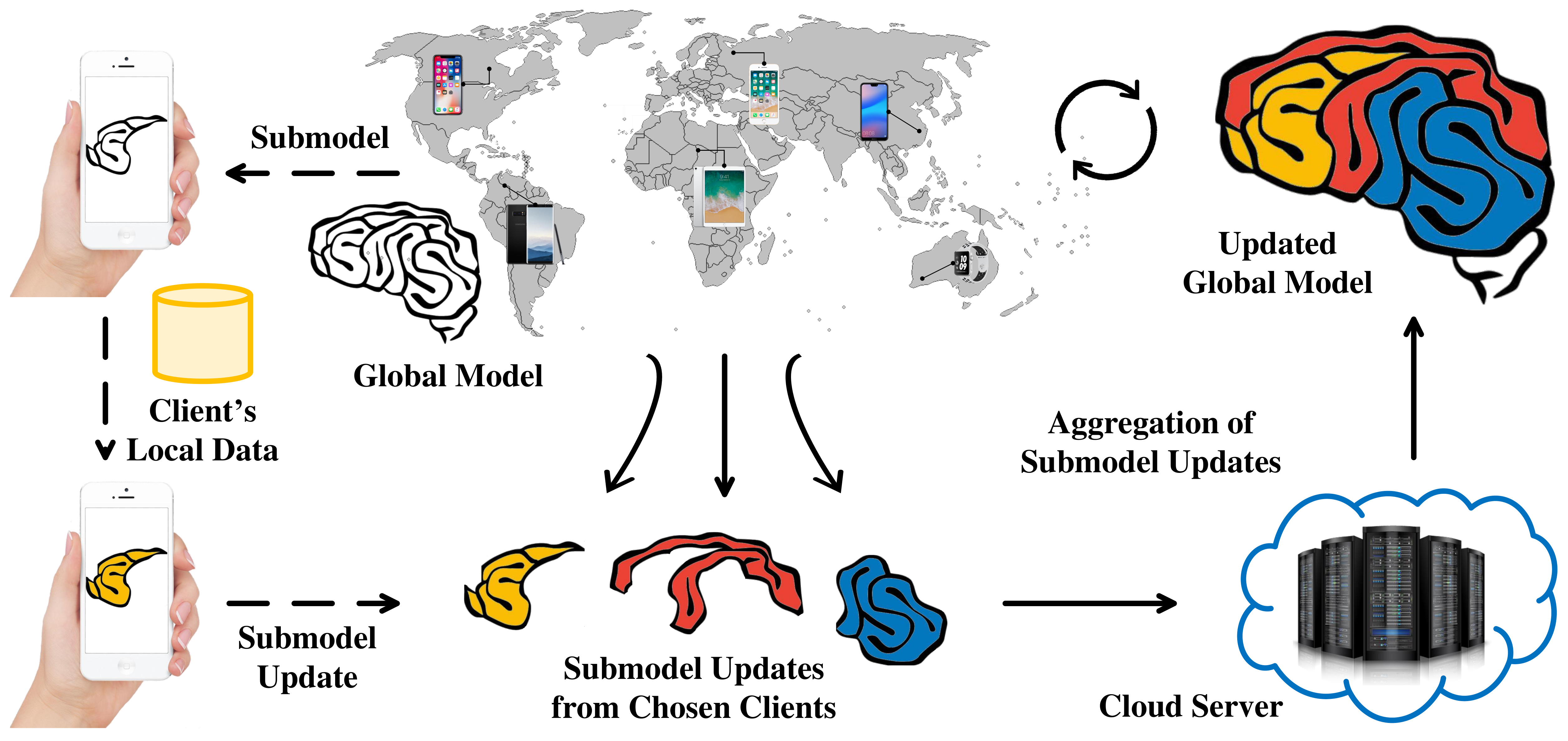}}
\caption{A visual comparison of conventional federated learning and our new federated submodel learning.}\label{fig:system:model}
\end{figure*}

The industrial scenario in Alibaba that motivated federated submodel learning is the desire to provide customized and accurate e-commerce recommendations for billion-scale clients while keeping user data on local devices.

Currently, the recommendation systems in Alibaba are cloud based and require the server cluster to collect, process, and store numerous user data. In addition, the deployed recommendation models follow a golden paradigm of embedding\footnote{In general, deep learning with a huge and sparse input space (\eg, e-commerce goods IDs, natural language texts, and locations) requires an embedding layer to first transform inputs into a lower-dimensional space~\cite{proc:sp19:feature:leakage:fl}. Additionally, the full embedding matrix tends to occupy a large proportion of the whole model parameters (\eg, $98.22\%$ in our evaluated DIN model and more than two-thirds in Gboard's CIFG language model~\cite{tc:arxiv18:gboard:prediction}).} and Multi-Layer Perceptron (MLP)~\cite{proc:16:wd}: user data are first encoded into high-dimensional sparse feature vectors, then embedded into low-dimensional dense vectors, and finally fed into fully connected layers. To improve accuracy, Deep Interest Network (DIN)~\cite{proc:alibaba:kdd18:din} introduces the attention mechanism to activate the user's historical behaviors, namely relative interests, with respect to the target item; Deep Interest Evolution Network (DIEN)~\cite{proc:alibaba:aaai20:dien} further extracts latent interests and monitors interest evolution through Gated Recurrent Unit (GRU) coupled with attention update gate; and Behavior Sequence Transformer (BST)~\cite{tc:alibaba:transformer} incorporates transformer to capture the sequential signals underlying the user's behavior sequence.

However, typical fields of user data involved in recommendation include user profile (\eg, user ID, gender, and age), user behavior (\eg, the list of visited goods IDs and relevant information, such as category IDs and shop IDs), and context (\eg, time, page number, and display position). More or less, these data fields are sensitive, and some clients who value security and privacy highly may refuse to share their data. In addition, according to the General Data Protection Regulation (GDPR), which was legislated by the European Commission and took effect on May 25, 2018, any institution or company is prohibited from uploading user data and storing it in the cloud without the explicit permissions from the European Union users~\cite{link:gdpr,proc:nips19:data:deletion}. Under such circumstances, refining the recommendation models and further providing accurate recommendations become urgent demands as well as thorny challenges in practice.


Federated learning, which decouples the ability to do machine learning from the need to upload and store data in the cloud, is a potential solution. However, the original framework of federated learning, proposed by Google researchers in~\cite{proc:aistats17:fedavg}, requires each client to download the full machine learning model for training and then to upload the update of the full model, which is impractical for resource-constrained clients in the context of complex deep learning tasks. For example, as the largest online consumer-to-consumer platform in China, Taobao (owned by Alibaba) has roughly two billion goods in total~\cite{proc:alibaba:kdd18:match}, which is far larger than the 10,000 word vocabulary in the natural language scenario of Google's Gboard~\cite{tc:arxiv18:gboard:prediction,tc:arxiv19:oovw,tc:arxiv19:gboard:emoji:prediction}. This implies that the full embedding matrix of goods has roughly two billion rows and roughly occupies 134GB of space, when the embedding dimension is 18, and each element adopts 32-bit representation. If each client directly uses the full matrix for learning, it inevitably incurs huge overheads, which are unacceptable and unaffordable for one billion Taobao users with smart devices. To improve efficiency, we observe that a certain user tends to browse, click, and buy a small number of goods, and thus just needs a tailored model, which can sharply reduce the overheads and is more practical for mobile clients. Continuing with the above example, if a Taobao user's historical data involve 100 goods, she only requires the corresponding 100 rows, rather than the entire two billion rows, of the embedding matrix. Based on this key observation, we propose a new framework of federated learning, called federated submodel learning, as follows.


\subsection{Framework of Federated Submodel Learning}

We plot the workflow of federated submodel learning in Fig.~\ref{fig:system:model:sub} and also provide the traditional federated learning in Fig.~\ref{fig:system:model:full} for an intuitive comparison.

In the beginning of one communication round, a cloud server first selects a certain number of eligible clients, typically end users whose mobile devices are idle, charging, and connected to an unmetered Wi-Fi network. This eligibility criteria is used to avoid a negative effect on the user experience, data usage, or battery life. Then, each chosen client downloads part of the global model as she requires, namely a submodel, from the cloud server. For example, in the e-commerce scenario above, a client's submodel mainly consists of the embedding parameters for the displayed and clicked goods in her historical data, as well as the parameters of the other network layers. Afterwards, the client trains the submodel over her private data locally. At the end of one round, the cloud server lets those chosen clients who are still alive upload the updates of their submodels and further aggregates the submodel updates to form a consensus update to the global model. Considering the convergencies of the global model at the cloud server and the submodels on clients, the above process is iterated for several rounds.


If each client leverages the full model rather than her required submodel for learning, federated submodel learning will degenerate to conventional federated learning. Compared with the conventional one, our new framework further decouples the ability to accomplish federated learning from the need to use the prohibitively large full model, which can dramatically improve efficiency. For example, in our evaluation, the size of a client's desired submodel is only $1.99\%$ of the full model's size. Thus, our framework is more practical for resource-constrained clients and deep learning tasks.

\subsection{Newly Introduced Privacy Risks}\label{sec:introduction:privacy:risk}

Just as every coin has two sides, federated submodel learning not only brings in efficiency but also introduces extra privacy risks. On one hand, compared with using the public full model in conventional federated learning, the download of a submodel and the upload of the submodel update would require each client to provide an index set as auxiliary information, specifying the ``position" of her submodel. However, the index set normally corresponds to the client's private data. For example, to specify the required rows of the embedding matrix in the e-commerce scenario, a client mainly needs to provide the goods IDs in her user data as the index set. Thus, the disclosure of a client's real index set to the cloud server can still be regarded as the leakage of the client's private data, breaking the tenet of federated learning. On the other hand, compared with the aligned full model in federated learning, each client only submits the update of her customized and highly differentiated submodel in federated submodel learning. As a result, the aggregation of updates with respect to a certain index can come from a unique client (\eg, with probability $86.7\%$ for 100 chosen clients in our evaluated Taobao dataset), which indicates that the cloud server not only can ascertain that the client has a certain index but also can learn her detailed update. These two kinds of knowledge both breach the client's private data. Further, such a privacy risk in e-commerce is more severe than that in natural language because compared with the vocabularies of different Gboard users, the goods IDs of different Taobao users are more differentiated. We will detail and visualize the preceding privacy risks in Section~\ref{sec:adversary:model} and Fig.~\ref{fig:adversary:model}.

\subsection{Fundamental Problems and Challenges}\label{sec:introduction:problems:challenges}

In essence, to mitigate the above privacy risks, we need to jointly solve two fundamental problems modeled from the processes of downloading a submodel and uploading a submodel update, respectively. One is how a client can download a row from a matrix, maintained by an untrusted cloud server, without revealing which row, alternatively the row index, to the cloud server. The other is how a client can modify a row of the matrix, still without revealing which row was modified and the altered content to the cloud server. Using the terminology from file system permissions, the first problem has a ``read-only" attribute, where the client only reads the file. In contrast, the second problem is in a ``write" mode, where the client can edit the file. Further incorporating the obscure requirement of two operations, the second problem appears more challenging than the first one. We now analyze these two problems in detail as follows.

We start with the first problem. One trivial method is that the client downloads the full matrix, as in conventional federated learning, and then extracts the required row locally. Although this method perfectly hides the fetched row index, it incurs significant communication cost, which can be unaffordable for resource-constrained mobile devices, especially when the matrix is huge, \eg, representing a deep neural network. To avoid downloading the full matrix, Private Information Retrieval (PIR)~\cite{proc:sp18:pir,proc:ccs18:pir:labeled:psi,proc:ccs18:pir:state} can be applied, which exactly matches our problem settings, including the read-only mode and the privacy preservation requirement of the retrieved elements. The state-of-the-art constructions of private information retrieval include Microsoft's SealPIR~\cite{proc:sp18:pir} and Labeled PSI~\cite{proc:ccs18:pir:labeled:psi} and Goolge's PSIR~\cite{proc:ccs18:pir:state}, where two Microsoft protocols have been deployed in its Pung private communication system~\cite{proc:osdi16:pung}. We note that another celebrated cryptographic primitive, called Oblivious Transfer (OT)~\cite{tc:iacr05:ot:robin}, is stronger than private information retrieval. It not only guarantees that the cloud server does not know which row the client has downloaded, as in private information retrieval, but also ensures that the client does not know the other rows of the matrix, which is instead not needed in practical federated submodel learning. Therefore, if we consider the first problem independently, private information retrieval may be a good choice.


We next dissect the second problem. For a concrete row of the full matrix, if clients modify this row one by one, the cloud server definitely knows those clients who modified this row and their detailed contents of modification. Thus, one feasible way is to first securely aggregate all the modifications without revealing any individual modification, and then apply the aggregate modification to the row of the full matrix once. In particular, such a guarantee can be provided by the secure aggregation protocol in~\cite{proc:ccs17:secure:agg} and some other schemes for oblivious addition, \eg, based on additively homomorphic cryptosystems~\cite{proc:eurocrypt99:paillier,proc:tcc05:boneh,proc:eurocrypt10:freeman}. With the secure aggregation guarantee, if more than one client participates in aggregation and at least one of their modifications is nonzero, then the cloud server cannot reveal which client(s) truly intend to modify this row and their detailed modifications. Further, a larger number of involved clients implies a stronger privacy guarantee. One extreme case is in conventional federated learning, which harshly lets all chosen clients in one communication round be involved, no matter whether they truly intend to modify this row or not. Thus, it can offer the best privacy guarantee. Nevertheless, considering each client needs to be involved for each row of the full matrix, it is too inefficient to be applicable in the large-scale deep learning context. Another extreme case is in federated submodel learning, which simply lets those clients who really intend to modify this row be involved. Hence, each chosen client only needs to be involved for those rows that she truly intends to modify, implying the best efficiency. However, different clients tend to modify highly differentiated or even mutually exclusive rows. For the joint modification with respect to some row, chances are high that only one client is involved. Under such a circumstance, the secure aggregation guarantee no longer works, which leaks the client's real intention and her detailed modification. In a nutshell, trivial solutions to the second problem cannot well balance or support tuning privacy and efficiency.


\subsection{Our Solution Overview and Major Contributions}


Jointly considering the above two fundamental problems and several practical issues, we propose a secure scheme for federated submodel learning. In our scheme, each chosen client generates three types of index sets locally: real, perturbed, and succinct. First, the real index is extracted from a client's private data and is kept secret from the other system participants, including the cloud server and any other chosen client. Second, the perturbed index set is used to interact with others in the download and upload phases. It is generated by applying randomized response twice with one memoization step between. Such a design, together with secure aggregation, allows the client to hold a self-controllable deniability against whether she really intends or does not intend to download some row and to upload the modification of this row, even if the client may be chosen to participate in multiple communication rounds. The strength of deniability is rigorously quantified using local differential privacy. Further, rather than trivially using the prohibitively large-scale full index set as the questionnaire of randomize response in every communication round, we identify a necessary and sufficient index set, namely the union of the chosen clients' real index sets. Considering the secrecy of each client's real index set, we propose an efficient and scalable Private Set Union (PSU) protocol based on Bloom filter, secure aggregation, and randomization, allowing clients to obtain the union under the mediation of an untrusted cloud server without revealing any individual real index set. In particular, private set union promises a wide range of applications but receives little attention. Due to unaffordable overheads, none of the existing protocols can be deployed in practice yet. Last, the succinct index set is derived from the intersection between the real and perturbed index sets, and it is used to prepare the data and submodel for local training.

We summarize our key contributions in this work as follows:
\begin{itemize}
\item To the best of our knowledge, we are the first to propose the framework of federated submodel learning and further to identify and remedy new privacy risks.
\item Our proposed secure scheme mainly features the properties of randomized response and secure aggregation to empower each client with a tunable deniability against her real intention of downloading the desired submodel and uploading its update, thus protecting her private data. As a moat, we designed an efficient and scalable private set union protocol based on Bloom filter and secure aggregation, which can be of independent and significant value in practice.
\item We instantiated with Taobao's e-commerce scenario, adopted Deep Interest Network (DIN) for recommendation, and implemented a prototype system. Additionally, we extensively evaluated over one month of Taobao data. The evaluation and analysis results demonstrate the practical feasibility of our scheme, as well as its remarkable advantages over the conventional federated learning framework in terms of model accuracy and convergency, communication, computation, and storage overheads. Specifically, when the number of chosen clients in one round is 100, compared with conventional federated learning, which diverges in the end, our scheme improves the highest Area Under the Curve (AUC) by 0.072. In addition, at the same security and privacy levels as conventional federated learning with secure aggregation, our scheme reduces $80.05\%$ of communication overhead on both sides of the client and the cloud server. Moreover, our scheme reduces $85.02\%$ (\resp, $45.43\%$) and $72.51\%$ (\resp, $63.77\%$) of computation (\resp, memory) overheads on the sides of the client and the cloud server, respectively. Furthermore, when the size of the full model scales further, it does not incur additional overhead to our scheme, but it prohibits conventional federated learning from being applied. Finally, for our private set union, the communication overhead per client is less than 1MB, and the computation overheads of the client and the cloud server are both less than 40s, even if the dropout ratio of the chosen clients reaches $20\%$.
\end{itemize}

\section{Related Work}

In recent years, federated learning has become an active topic in both academic and industrial fields. In this section, we briefly review some major focuses and relevant work as follows. For more related work, we direct interested readers to the surveys written by Li~\et~\cite{tc:arxiv19:cmu:fl:survey} and Yang~\et~\cite{jour:tist19:fl:qiang:yang:survey}.


First and most important is to identify and address security and privacy issues of federated learning. Bonawitz~\et~\cite{proc:ccs17:secure:agg} proposed a secure, communication-efficient, and failure-robust aggregation protocol in both honest-but-curious and active adversary settings. It can ensure that the untrusted cloud server learns nothing but the aggregate (or mathematically, the sum) of the model updates contributed by chosen clients, even if part of clients drop out during the aggregation process. To bound the leakage of a certain client's training data from her individual model update, several differentially private mechanisms were proposed. McMahan~\et~\cite{proc:iclr18:dp:lstm} offered client-level differential privacy for recurrent language models based on the celebrated moments account scheme in~\cite{proc:ccs16:dpdp}. Here, the moments account allows the release of all intermediate results during the training process, particularly the gradients per iteration; keeps track of privacy loss in every iteration; and provides a tighter compositive/cumulative privacy guarantee. However, in the practical federated learning scenario, only the model update after multiple iterations/epochs is revealed, whereas all intermediate gradients are hidden. Specific to this case, Feldman~\et~\cite{proc:focs18:dp:epoch} analyzed the detailed amplification effect of hiding intermediate results on differential privacy. In contrast to these defense mechanisms, Bagdasaryan~\et~\cite{tc:arxiv18:cornell:backdoor:fl} developed a model replacement attack launched by malicious clients to backdoor the global model at the cloud server. Melis~\et~\cite{proc:sp19:feature:leakage:fl} exploited membership and property inference attacks to uncover features of the clients' training data from model updates.


Second is to improve the communication efficiency, especially the expensive and limited up-link bandwidth for mobile clients. To overcome this bottleneck, two types of solution methods have been proposed in general. One is to reduce the total number of communication rounds between the cloud server and the clients. A pioneering work is the federated averaging algorithm proposed by McMahan~\et~\cite{proc:aistats17:fedavg}. Its key principle is to let each client locally train the global model for multiple epochs, and then upload the model update. Thus, it is more communication efficient than the common practice of conventional distributed learning to exchange gradients per iteration in datacenter-based scenarios. The other complementary way is to further reduce the size of the transmitted message in each communication round, particularly through compressing model updates. Typical compression techniques include sparsification, subsampling, and probabilistic quantization coupled with random rotation. For example, after quantization, the original float-type elements of the update of the global model can be encoded as integer-type values with a few bits~\cite{tc:arxiv16:konecny:commeff,proc:icml17:kquantization}. Considering the compressed model updates are discrete, while classic differentially private deep learning mechanisms, hinging on the Gaussian mechanism, only support continuous inputs, Agarwal~\et~\cite{proc:nips18:cpsgd} proposed a Binomial mechanism to guarantee differential privacy for one iteration while enjoying communication efficiency. Another effective approach to improving communication efficiency is to first apply dropout strategies to the global model, and then let clients train over the same reduced model architecture~\cite{tc:arxiv18:fl:comm}. As a result, the downloaded model and the uploaded model update can be compressed in terms of dimension.


Third is from learning theory. The federated learning framework has several atypical characteristics: non independent and identically distributed (non-iid) and unbalanced data distributed over numerous clients with limited and unstable network connections. Such statistical heterogeneity and existence of stragglers make most existing analysis techniques for iid data infeasible and pose significant challenges for designing theoretically robust and efficient learning algorithms. The federated averaging algorithm mentioned above, as a cornerstone of federated learning, empirically shows its effectiveness in some tasks, but was observed to diverge for a large number of local epochs in~\cite{proc:aistats17:fedavg}. More specifically, it lets multiple chosen clients run mini-batch Stochastic Gradient Descent (SGD) in parallel, and then lets the cloud server periodically aggregate the model updates in a weighted manner, where weights are proportional to the sizes of the clients' training sets. Recently, Yu~\et~\cite{proc:aaai19:alibaba:fl} and Li~\et~\cite{tc:arxiv19:zhihuazhang:convergence} advanced the convergency analysis of federated averaging by imposing smooth and bounded assumptions on the loss function. The follow-up work~\cite{proc:icml19:alibaba:fed} further presented a momentum extension of parallel restarted SGD, which is compatible with federated learning. Different from the above work, Smith~\et~\cite{proc:nips17:smith:fed:multitask} focused on learning separate but related personalized models for distinct clients by leveraging multitask learning for shared representation. Chen~\et~\cite{tc:arxiv18:fl:meta:huawei} instead adopted meta-learning to enable client-specific modeling, where clients contribute information at the algorithm level rather than the model level to help train the meta-learner. Mohri~\et~\cite{proc:icml19:agnostic:fl} considered an unfairness issue that the global model can be unevenly biased toward different clients. They thus proposed a new agnostic federated learning framework where the global model can be optimized for any possible target distribution, which is formed via a mixture of client distributions. Eichner~\et~\cite{proc:icml19:semi:cyclic:sgd} captured data heterogeneity in federated learning, particularly cyclic patterns, and offered a pluralistic solution for convex objectives and sequential SGD.

Fourth is regarding production and standardization. Google has deployed federated learning in its Android keyboard, called Gboard, to polish several language tasks, including next-word prediction~\cite{tc:arxiv18:gboard:prediction}, query suggestion~\cite{tc:arxiv18:gboard:suggestion}, out-of-vocabulary words learning~\cite{tc:arxiv19:oovw}, and emoji prediction~\cite{tc:arxiv19:gboard:emoji:prediction}. In particular, the query suggestion used logistic regression as the triggering model for on-device training to determine whether the candidate suggestion should be shown or not. In addition, the other three tasks leveraged a tailored Long Short-Term Memory (LSTM) Recurrent Neural Network (RNN), called Coupled Input and Forget Gate (CIFG). In~\cite{proc:sysml19:fed}, Google's team also detailed their initial system design and summarized practical deployment issues, such as irregular device availability, unreliable network connectivity and interrupted execution, orchestration of lock-step execution across heterogonous devices, and limited device storage and computation resources. They also pointed out some future optimization directions in bias, convergence time, device scheduling, and bandwidth. To facilitate open research, Google integrated a federated learning simulation interface into its deep learning framework, called TensorFlow Federated~\cite{link:tf:fl}. However, this open-source module lacks several core functionalities, \eg, secure and privacy preserving mechanisms, on-device training, socket communication between cloud server and clients, task scheduling, and dropout/exception handling. These significantly suppress federated learning related productions in other commercial companies. Caldas~\et~\cite{tc:arxiv18:fl:leaf:benchmark} released a benchmark for federated learning, called LEAF. Currently, LEAF comprises some representative datasets, evaluation metrics, and a referenced implementation of federated averaging.

Parallel to existing work, where clients use the same (simplified) global model for learning, we propose a novel federated submodel learning framework for the sake of scalability. Under this framework, we identify and remedy new security and privacy issues, due to the dependence between the position of a client's desired submodel and her private data as well as the misalignment of clients' submodel updates in aggregation.


\section{Preliminaries}

\begin{figure*}[t]
\centering
\subfigure[Federated Learning with Secure Aggregation]{\label{fig:adversary:model:full}
\includegraphics[width=0.95\columnwidth]{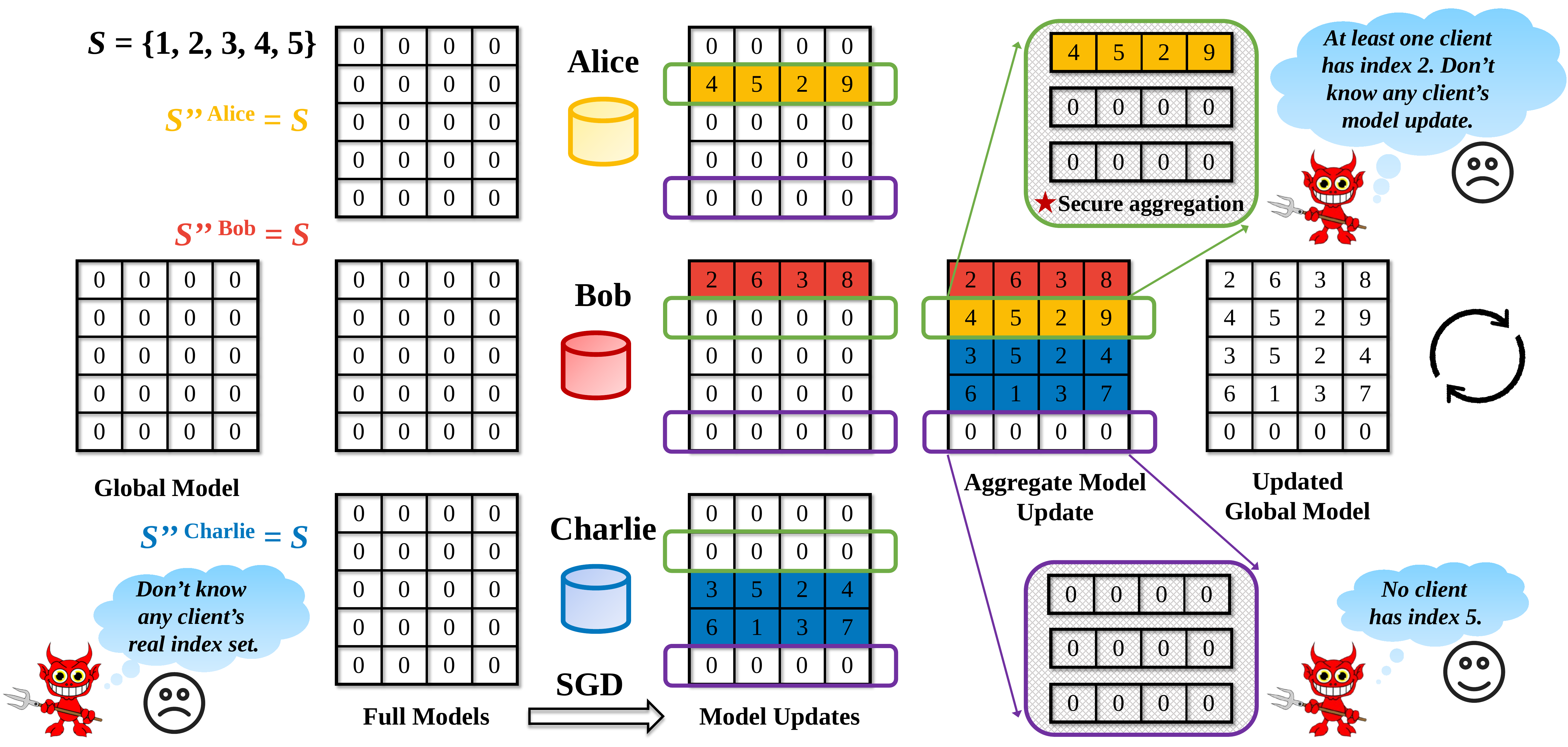}}
\subfigure[Federated Submodel Learning with Secure Aggregation]{\label{fig:adversary:model:sub}
\includegraphics[width=0.95\columnwidth]{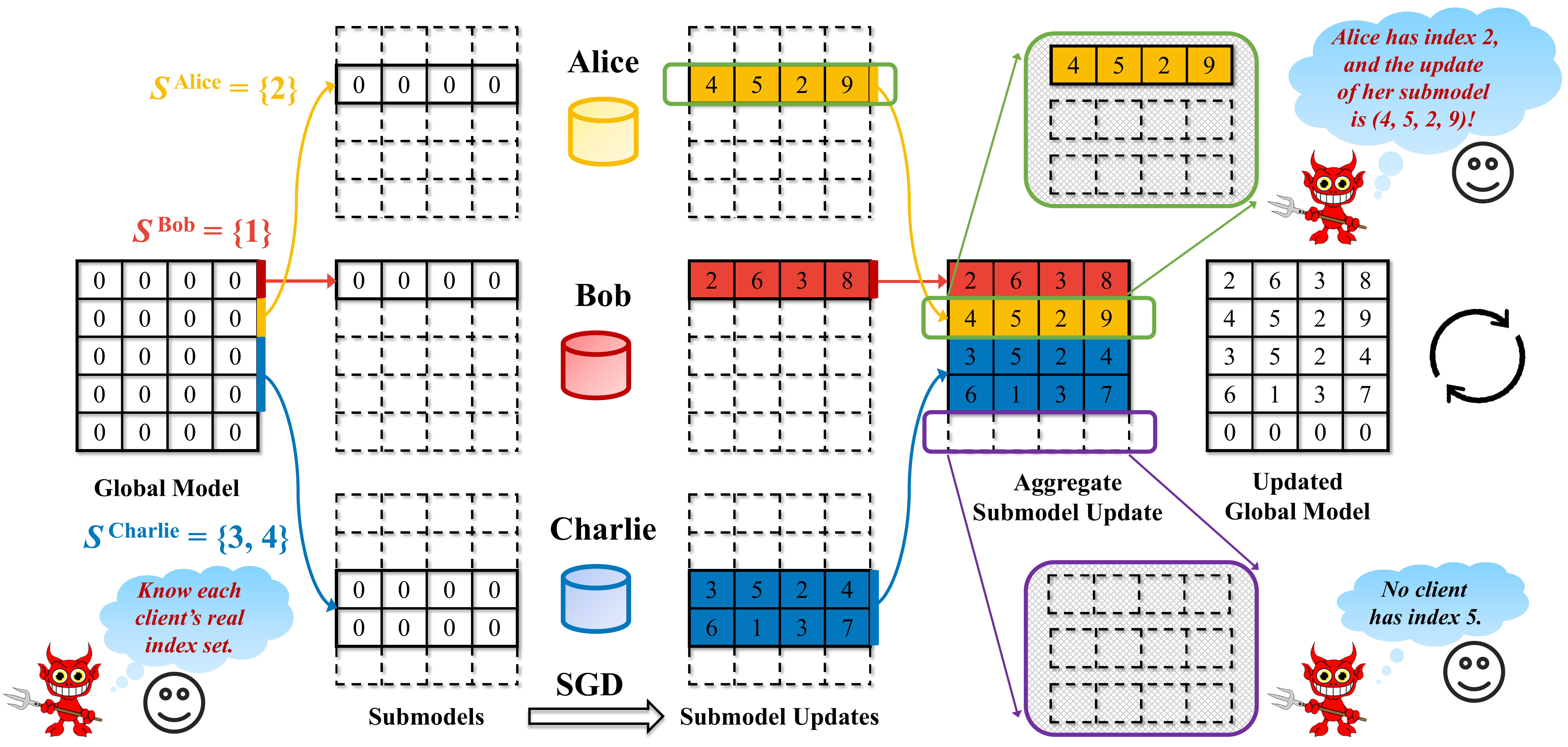}}
\caption{Federated submodel learning with secure aggregation can still leak a concrete client's real indices and her detailed updates to the cloud server, compared with conventional federated learning with secure aggregation. The rounded rectangle colored in dark gray denotes the process of secure aggregation~\cite{proc:ccs17:secure:agg}, where the cloud server, as the aggregator, only obtains the sum of vectors from multiple clients but does not know any individual client's vector. Additionally, the row in a table with dashed lines indicates that the client does not download, train over, or upload this row.
}\label{fig:adversary:model}
\end{figure*}

\begin{table}[!t]
	\caption{Frequently used notations and abbreviations.} \label{tab:notation}
	\centering
	\resizebox{\columnwidth}{!}{
		\begin{tabular}[t]{l|p{6.2cm}}
			\toprule
			Notation & Remark \\
			\midrule\midrule
            $\mathbf{W}^{m \times d}$ & Global/Full model at the cloud server, denoted by a matrix with $m$ rows and $d$ columns\\
            $\mathcal{S} = \{1, 2, \ldots, m\}$ & Full row index set of $\mathbf{W}$\\
            $\mathcal{C}, \vert\mathcal{C}\vert = n$ & The set of $n$ clients chosen by the cloud server in one communication round, the cardinality of $\mathcal{C}$\\
            $\hat{\mathcal{C}} \subset \mathcal{C}$ & The up-to-date set of clients who are alive throughout the communication round\\
            $i \in \mathcal{C}$ & A chosen client $i$\\
            $\mathcal{S}^{(i)} \subset \mathcal{S}$ & Client $i$'s real index set that corresponds to local data and specifies truly required rows of $\mathbf{W}$\\
            $\mathcal{S}''^{(i)}$ & A perturbed index set of client $i$, to download the submodel from the cloud server and to securely upload the update of the submodel to the cloud server\\
            $\mathbf{W}_{\mathcal{S}''^{(i)}}$ & Client $i$'s downloaded submodel\\
            $\mathbf{W}_{\mathcal{S}^{(i)}\bigcap\mathcal{S}''^{(i)}}$ & Client $i$'s succinct submodel for local training\\
            $\Delta\mathbf{W}_{\mathcal{S}^{(i)}\bigcap\mathcal{S}''^{(i)}}$ & Client $i$'s succinct submodel update\\
            $\Delta\mathbf{W}_{\mathcal{S}''^{(i)}}$ & Client $i$'s uploaded submodel update by padding zero vectors to the succinct submodel update\\
            $\epsilon$ & A privacy level/budget of local differential privacy\\
            $\mathbf{b}, \beta, h, \phi$ & A Bloom filter with $\beta$ bits and $h$ hash functions, representing/accommodating a set of $\phi$ elements\\
            $l$ & Dimension of vector in the secure aggregation protocol\\
            $p_1^{(i)}, p_2^{(i)}, p_3^{(i)}, p_4^{(i)}$ & Client $i$'s probability parameters to generate $\mathcal{S}''^{(i)}$\\
            $p_5^{(i)} = p_1^{(i)}(p_3^{(i)} - p_4^{(i)}) + p_4^{(i)}$ & The probability that an index in client $i$'s real index set will fall into her perturbed index set\\
            $p_6^{(i)} = p_2^{(i)}(p_3^{(i)} - p_4^{(i)}) + p_4^{(i)}$ & The probability that an index not in client $i$'s real index set will fall into her perturbed index set\\
            $p_{7}$ & The probability of the cloud server ascertaining that an index belongs to some client's real index set and also learning her detailed update with respect to this index from the securely aggregated submodel update\\
            $p_{8}$ & The probability of the cloud server ascertaining that an index does not belong to some client's real index set from the securely aggregated submodel update\\
            $s$   & The expected cardinality of each client's real index set\\
            $\mathbb{Z}_R = \{0, 1, \ldots, R-1\}$ & The least residue system modulo $R$\\
            $\gamma$ & A level of stochastic quantization mechanism\\
            FL     & Conventional federated learning\\
            SFL    & Secure federated learning, namely conventional federated learning with secure aggregation\\
            SFSL   & Secure federated submodel learning\\
            CPP    & Choice of probability parameters\\
            SA     & Secure aggregation\\
			\bottomrule
		\end{tabular}
	}
\end{table}


In this section, we elaborate on the privacy risks sketched in Section~\ref{sec:introduction:privacy:risk} and formally define the corresponding security requirements. We also review some existing building blocks.



We first introduce some necessary notations. For the sake of clarity, frequently used notations and abbreviations throughout this paper are also listed in Table~\ref{tab:notation}. We use a two-dimensional matrix with $m$ rows and $d$ columns to represent the global/full model, denoted as $\mathbf{W}$. Such a matrix-based representation not only suffices for the recommendation models used in Alibaba but also can easily degenerate to a widely used vector-based representation~\cite{proc:ccs17:secure:agg,proc:nips18:cpsgd}, by setting the number of columns $d$ to 1. Additionally, we let $\mathcal{S} = \{1, 2, \ldots, m\}$ denote the entire row index set of $\mathbf{W}$. Moreover, we let $\mathcal{C}$ denote those clients who are selected by the cloud server to participate in one communication round of federated submodel learning. For a chosen client $i \in \mathcal{C}$, we let $\mathcal{S}^{(i)} \subset \mathcal{S}$ denote her real index set, which implies that the user data of client $i$ involves the rows in $\mathbf{W}$ with indices $\mathcal{S}^{(i)}$.

\subsection{Details on Privacy Risks and Security Requirements}\label{sec:adversary:model}

We now expand on two kinds of privacy leakages that the federated submodel learning brings in, compared with conventional federated learning. We provide Fig.~\ref{fig:adversary:model} for illustration. We here adopt an {\em honest-but-curious} security model, in which the cloud server and all clients follow the designed protocol, but try to glean sensitive information about others.



The first kind of privacy leakage is the disclosure of a client's real index set, which specifies the position of a submodel and implies the client's private data, to the cloud server. For example, each row of the embedding matrix for goods in the recommendation model is linked with a certain goods ID, which indicates that a client's real index set, specifying her required rows of the embedding matrix, is in fact the goods IDs in her private data. Similarly, when federated submodel learning is applied to the natural language scenario (\eg, next-word prediction in Gboard), a client's real index set to locate her wanted parameters of word embedding is actually the vocabulary extracted from her typed texts. Thus, the disclosure of a client's real index set to the cloud server is still regarded as the leakage of the client's private data. In contrast, for conventional federated learning, each client essentially uses the full index set, which is public to the cloud server and all other clients, and does not reveal any private information.

The second kind of privacy leakage is from the aggregation of misaligned submodel updates, where the cloud server may not only know that a certain client has a concrete index but also learn her detailed update with respect to this index. In addition to the fact that the real index reveals a client's private data, the client's individual submodel update can still memorize or even allow reconstruction of her private data, namely ``model inversion" attack~\cite{proc:usenix14:fredrikson,proc:ccs15:fredrikson:model:inversion,jour:ijsn15:evasion,proc:sp19:feature:leakage:fl,proc:nips19:leakage}. To conceal a client's individual update in conventional federated learning, the secure aggregation protocol~\cite{proc:ccs17:secure:agg} can be applied, which allows the cloud server to obtain the sum of multiple vectors without learning any individual vector. As shown in Fig.~\ref{fig:adversary:model:full}, with respect to index $2$, Alice submits the update, denoted by the vector $(4, 5, 2, 9)$, whereas Bob and Charlie submit two zero vectors. The secure aggregation protocol can guarantee that the cloud server only obtains the sum of three vectors, \ie, $(4, 5, 2, 9)$, but does not know the content of any individual vector. This further implies that from the aggregate result, the cloud server can merely infer that at least one client has index $2$, but cannot identify which client(s). Such a functionality is essentially analogous to anonymization. In a nutshell, the zero updates from Bob and Charlie function as two shields of Alice. However, in federated submodel learning, due to the differentiation and misalignment of clients' submodels, the ``zero" shields from other clients vanish, and the aggregation of updates with respect to a certain index can come from one unique client, making secure aggregation ineffective. For example, in Fig.~\ref{fig:adversary:model:sub}, only Alice who has index 2 submits the update $(4, 5, 2, 9)$, whereas Bob and Charlie submit nothing. Without the blindings from Bob and Charlie, the cloud server not only knows that Alice has index $2$ while Bob and Charlie do not have but also learns Alice's detailed update $(4, 5, 2, 9)$.



Given the two kinds of privacy leakages above, we define the corresponding security requirements. First, for the disclosure of real index sets when clients interact with the cloud server, we consider that each client should have {\em plausible deniability} of whether a certain index is or is not in her real index set. To measure the strength of plausible deniability, we adopt local differential privacy, which is a variant of standard differential privacy in the local setting. Specifically, the perturbation in local differential privacy is performed by clients in a distributed manner, rather than relying on a data curator, as a trusted authority to conduct centralized perturbation in differential privacy. Thus, the privacy of an individual client's data is not only preserved from external attackers but also from the untrusted data curator, \eg, the cloud server in our context. Due to its intriguing security properties, local differential privacy for various population statistics has recently received significant industrial deployments (\eg, in Google~\cite{proc:ccs14:rappor,jour:popets:rappor2}, Apple~\cite{jour:apple17:ldp}, and Microsoft~\cite{proc:nips17:ldp:ding:microsoft}), as well as lasting academic attention~\cite{proc:focs08:ldp,proc:focs13:ldp,proc:stoc15:ldp,proc:ccs16:ldp,proc:nips17:ldp,proc:usenixs17:ldp:wang,proc:sp18:ldp:wang,proc:ccs18:ldp:ninghui,proc:soda19:ldp:longitudinal,proc:sigmod19:wang:ding,tc:ldp:wang}. We now present the formal definition of local differential privacy as follows:

\begin{definition}[Local Differential Privacy]\label{def:ldp}
A randomized mechanism $M$ satisfies $\epsilon$-local differential privacy, if for any pair of inputs from a client, denoted as $x$ and $y$, and for any possible output of $M$, denoted as $z$, we have
\begin{align*}
\frac{\mathrm{Pr}\left(M\left(x\right) = z\right)}{\mathrm{Pr}\left(M\left(y\right) = z\right)} \leq \exp\left(\epsilon\right),
\end{align*}
where $\epsilon$ is a privacy budget controlled by the client. A smaller $\epsilon$ offers a better privacy guarantee.
\end{definition}

\noindent Intuitively, the above definition says that the output distribution of the randomized mechanism does not change too much, given distinct inputs from the client. Thus, local differential privacy formalizes a sort of plausible deniability: no matter what output is revealed, it is approximately equally as likely to have come from one input as any other input. In addition, when local differential privacy applies to obscure the membership of a certain index in federated submodel learning, the inputs and the outputs are boolean values, where possible inputs (\resp, outputs) are two states: a certain index ``in" or ``not in" a client's real (\resp, revealed) index set. Moreover, we can check that conventional federated learning provides the strongest deniability, where the level of local differential privacy is $\epsilon = \ln(1/1) = 0$ for each client. The reason is that no matter whether an index is or is not in a client's real index set (different inputs), this index will definitely be revealed (the same output). In contrast, federated submodel learning provides the weakest deniability, where the level of local differential privacy is $\epsilon = \ln(1/0) = \infty$ for each client, because if an index is in (\resp, not in) a client's real index set (different inputs), this index will definitely (\resp, definitely not) be revealed, \ie, the output with probability 1 (\resp, 0).



Second, direct secure aggregation of submodel updates is the most efficient but insecure case, which can leak whether some client has a certain index as well as her detailed update. In contrast, the other extreme case is conventional federated learning with secure aggregation, which is most secure but inefficient. Specifically, all participating clients upload the full model updates, which can perfectly prevent privacy leakages due to the misalignment of customized submodels. To enable clients to tune privacy and efficiency in a fine-grained manner, we define a client-controllable privacy protection mechanism for submodel updates aggregation.
\begin{definition}\label{def:client:control:privacy}
A privacy protection mechanism for submodel updates aggregation is client controllable, if it enables participating clients to determine the probabilities of the following two complementary events: From the securely aggregated submodel update,
\begin{itemize}
 \item {Event 1:} the cloud server ascertains that an index belongs to some client's real index set and also learns her detailed update with respect to this index;
 \item {Event 2:} the cloud server ascertains that an index does not belong to some client's real index set.
\end{itemize}
\end{definition}

\noindent We note that revealing the states of some clients having and not having a certain index should both be regarded as privacy leakages. Furthermore, when the above definition applies to federated learning, and if at least two clients participate in aggregation, the probability of Event 1 is 0, and the probability of Event 2 is still 0 for those indices within the union of the chosen clients' real index sets. For an index outside the union, \eg, index $5$ shown in Fig.~\ref{fig:adversary:model:full}, the probability of Event 2 is approaching 1. The reason is that from the aggregate zero vector, the cloud server almost ascertains that all clients do not have this index, despite of some rare cases (\eg, Alice and Bob submit two vectors of elements differing in signs, and Charlie submits a zero vector).

\subsection{Building Blocks}
We review randomized response, secure aggregation, and Bloom filter underlying our design.


\subsubsection{Randomized Response}
Randomized response, due to Warner in 1965~\cite{jour:jasa65:randomized:response}, is a survey technique in the social sciences to collect statistical information about illegal, embarrassing, or sensitive topics, where the respondents want to preserve privacies of their answers. A classical example for illustrating this technique is the ``Are you a member of the communist party?" question. For this question, each respondent flips a fair coin in secret and tells the truth if it comes up tails; otherwise, she flips a second coin and responds ``Yes" if heads and ``No" if tails. Thus, a communist (\resp, non-communist) will answer ``Yes" with probability $75\%$ (\resp, $25\%$) and ``No" with probability $25\%$ (\resp, $75\%$).

The intuition behind randomized response is that it provides plausible deniability for both ``Yes" and ``No" answers. In particular, a communist can contribute her response of ``Yes" to the event that the first and second coin flips were both heads, which occurs with probability $25\%$. Meanwhile, a non-communist can also contribute her response of ``No" to the event that the first coin is heads and the second coin is tails, which still occurs with probability $25\%$. Furthermore, the plausible deniability of randomized response can be rigorously quantified by local differential privacy. As analyzed in~\cite{jour:2014:dwork:dp,proc:ccs14:rappor}, for a one-time response, each respondent has local differential privacy at the level $\epsilon = \ln(75\% / 25\%) = \ln 3$, irrespective of any attacker's prior knowledge.

\subsubsection{Secure Aggregation}\label{sec:secure:aggregation}

An individual model update may leak a client's private data under the notorious model inversion attack. Nevertheless, to update the global model in federated learning, the cloud server does not need to access any individual model update and only requires the aggregate, basically the sum, of multiple model updates. For example, if $n$ clients participate in the aggregation, denoted as $\mathcal{C}$, where client $i \in \mathcal{C}$ holds a vector $\Delta\mathbf{w}^{(i)} \in \mathbb{Z}^l$ of dimension $l$, the cloud server should just obtain the sum $\sum_{i \in \mathcal{C}} \Delta\mathbf{w}^{(i)}$, while maintaining each individual $\Delta\mathbf{w}^{(i)}$ in secret. For this purpose and the characteristics of mobile devices, particularly limited and unstable network connections and common dropouts, Google researchers proposed a secure aggregation protocol in~\cite{proc:ccs17:secure:agg}. Implied by the functionality of oblivious addition, secure aggregation in federated learning can further ensure that even if the model inversion attack succeeds, the attacker (\eg, the honest-but-curious or active adversary cloud server, or an external intruder) may only infer that a group of clients has a certain data item but cannot identify which concrete client. This functionality is similar to anonymization. In what follows, we briefly review the secure aggregation protocol from communication settings, technical intuitions, scalability, and efficiency.

First, we introduce its communication settings. During the aggregation process, a client can neither establish direct communication channels with other clients nor natively authenticate other clients. However, each client has a secure (private and authenticated) channel with the cloud server. Thus, if one client intends to exchange messages with other clients, she needs to hinge on the cloud server as a relay. In addition, to guarantee confidentiality and integrity against the mediate cloud server, client-to-client messages should be encrypted with symmetric authenticated encryption, where the secret key is set up through Diffie-Hellman key exchange between two clients. Moreover, to defend active adversaries, a digital signature scheme is required for consistency checks. These basic settings make the secure aggregation protocol different from other relevant work about oblivious addition~\cite{proc:eurocrypt99:paillier,proc:tcc05:boneh,proc:eurocrypt10:freeman}, or, more generally, secure multiparty computation~\cite{stoc87:mpc,proc:stoc88:mpc,proc:usenix10:sepia,proc:nsdi17:prio}, which requires direct peer-to-peer communication between clients; assumes the availability of multiple noncolluding cloud servers; or resorts to a trusted third party for key generation and distribution.


Second, we outline the technical intuitions behind secure aggregation. Each client doubly masks her private data, including a self mask and a mutual mask. Here, the self mask is chosen by the client, whereas the mutual mask is agreed on with the other clients through Diffie-Hellman key exchange and is additively cancelable when summed with others. Considering that some clients may drop out at any point, their masks cannot be canceled. To handle this problem, each client uses a threshold secret sharing scheme to split her private seed of a Pseudo-Random Number Generator (PRNG) for generating the self mask as well as her private key for generating the mutual mask, and then distribute the shares to the other clients. As long as some minimum (no less than the threshold) number of clients remain alive, they can jointly help the cloud server remove the self masks of live clients and the mutual masks between dropped and live clients.


\begin{table}[t]
\caption{Complexities of the secure aggregation protocol in the honest-but-curious setting.} \label{tab:secure:aggregation:overhead}
\centering
\resizebox{\columnwidth}{!}{
\begin{tabular}[t]{l|ccc}
\toprule
             & Communication  &  Computation    &  Storage \tabularnewline
\midrule\midrule
Client       & $O(n + l)$     & $O(n^2 + nl)$   & $O(n + l)$ \tabularnewline
Server       & $O(n^2 + nl)$  & $O(n^2l)$       & $O(n^2 + l)$ \tabularnewline
\bottomrule
\end{tabular}
}
\end{table}

Third, we present the scalability and efficiency of the secure aggregation protocol. We list its communication, computation, and storage complexities in Table~\ref{tab:secure:aggregation:overhead}, where $n$ is the number of clients involved in the aggregation, and $l$ denotes the number of data items held by each client or the dimension of her data vector. We can see that this protocol is quite efficient for large-scale data vectors, especially from communication overhead, and thus can apply to mobile applications. In particular, as reported in~\cite{proc:ccs17:secure:agg}, when $2^{14}$ clients are involved in the aggregation, and each client has $2^{24}$ $16$-bit values, the communication overhead of the secure aggregation protocol expands $1.98\times$ over sending data in the clear.


\begin{figure*}[t]
\centering
\includegraphics[width = 1.9\columnwidth]{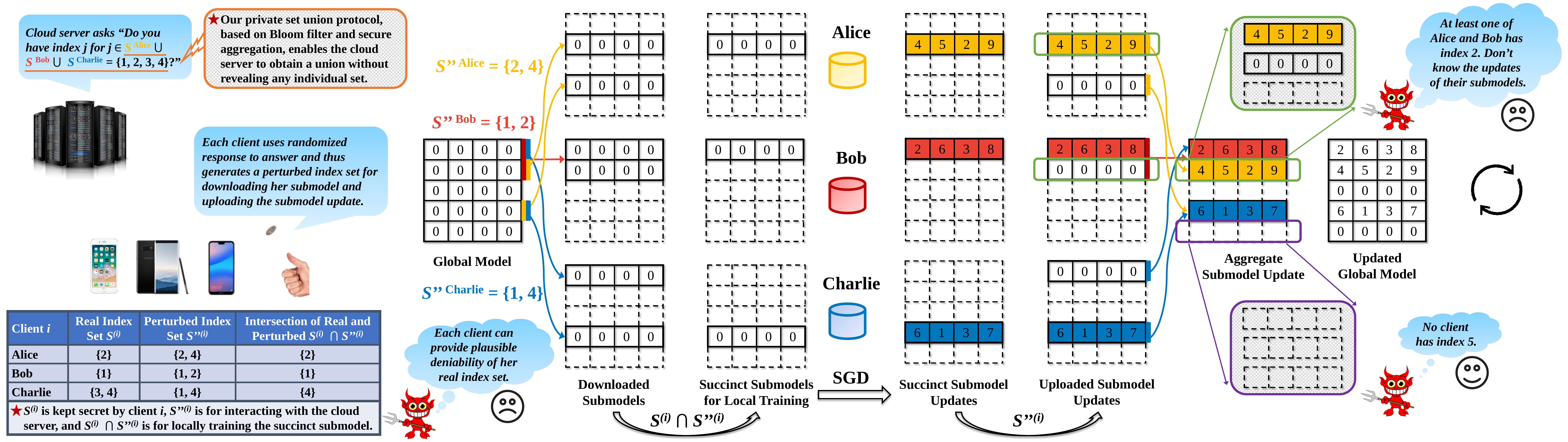}
\caption{An illustration of the design rationale of secure federated submodel learning.}\label{fig:design:rationale}
\end{figure*}

\subsubsection{Bloom Filter}\label{sec:bloom:filter}

Bloom filter, conceived by Bloom in 1970~\cite{jour:cacm1970:bloom:filter}, is a space-efficient probabilistic data structure to represent a set whose elements come from a huge domain. In addition, when testing whether an element is a member of the set, a false positive is possible, but a false negative is impossible. In other words, an element that is diagnosed to be present in the set possibly does not belong to the set in reality, and an element that is judged to be not present definitively does not belong to the set. We describe its technical details and properties as follows.

A Bloom filter is a $\beta$-length bit vector initially set to 0, denoted as $\mathbf{b}$. In addition, it requires $h$ different independent hash functions. The output range of these hash functions is $\{1, 2, \ldots, \beta\}$, which corresponds to the $\beta$ positions of the Bloom filter. To represent a set of $\phi$ elements, we apply $h$ hash functions to each element and set the Bloom filter at the positions of $h$ hash values to 1. In the membership test phase, to check whether an element belongs to the set, we simply check the Bloom filter at the positions of its $h$ hash values. If any of the bits at these positions is 0, the element is definitely not in the set. If all are 1, then either the element is in the set, or the bits have by chance been set to 1 during the insertion of other elements, resulting in a false positive. Specifically, the false positive rate ($\mathrm{FPR}$) of a Bloom filter depends on the length of Bloom filter $\beta$; the number of hash functions $h$; and the cardinality of set $\phi$. According to~\cite{jour:tmc03:bloom:filter,jour:im04:bloom:filter}, its detailed formula is given as
\begin{align*}
\mathrm{FPR} = \left(1 - \left(1 - \frac{1}{\beta}\right)^{h\phi}\right)^{h} \approx \left(1 - \exp\left(-\frac{h\phi}{\beta}\right)\right)^{h}.
\end{align*}
Given $\beta$ and $\phi$, to minimize the false positive rate, the optimal number of hash functions is
\begin{align*}
h = \ln 2 \frac{\beta}{\phi}.
\end{align*}
In addition, given $\phi$ and assuming the optimal number of hash functions $h$ is used, to achieve a desired false positive rate $\mathrm{FPR}$, the optimal length of Bloom filter should be
\begin{align}
\beta = -\frac{\phi \ln \mathrm{FPR}}{\left(\ln 2\right)^2}.\label{eq:bloom:filter:optimal:length}
\end{align}
Thus, the optimal number of bits per the set's element is
\begin{align*}
\frac{\beta}{\phi} = -\frac{\ln \mathrm{FPR}}{\left(\ln 2\right)^2},
\end{align*}
and the corresponding number of hash functions is
\begin{align*}
h = -\frac{\ln \mathrm{FPR}}{\ln 2}.
\end{align*}
The above deductions mean that for a given false positive rate, the length of a Bloom filter is proportional to the size of the set being filtered, while the required number of hash functions only relies on the target false positive rate.

We further introduce an appealing property of Bloom filter when performing union over the underlying sets. To represent sets, denoted as $\forall i \in \mathcal{C}, \mathcal{S}^{(i)}$, we use Bloom filters with the same length and the same hash functions, denoted as $\forall i \in \mathcal{C}, \mathbf{b}^{(i)}$. Then, the union of these sets, \ie, $\bigcup_{i \in \mathcal{C}} \mathcal{S}^{(i)}$, can be represented by a Bloom filter, which performs bitwise OR operations over the Bloom filters, \ie, $\bigvee_{i \in \mathcal{C}} \mathbf{b}^{(i)}$. Such a union operation is lossless (implying that the false positive rate remains unchanged) in the sense that the resulting Bloom filter is the same as the Bloom filter created from scratch using the union of these sets. In addition, because the Bloom filter needs to accommodate the union of sets, the parameter $\phi$, denoting the cardinality of the set, should be determined by estimating the cardinality of the union $\bigcup_{i \in \mathcal{C}} \mathcal{S}^{(i)}$ rather than that of each individual set $\mathcal{S}^{(i)}$.

We finally present a generalized version of Bloom filter, called counting Bloom filter~\cite{proc:sigcomm98:cbf}. It consists of $\beta$ counters/integers rather than $\beta$ bits and can represent a multiset, where an element can occur more than once. The main difference between counting Bloom filter and Bloom filter lies in that when representing an element, we increment the counters by one at the positions of its $h$ hash values. Thus, compared with the membership test in Bloom filter, counting Bloom filter further supports more general counting queries of a given element, \ie, whether the count number of a given element in a multiset is smaller than a certain threshold.


\section{Design of Secure Federated Submodel Learning}

In this section, we present the design rationale and the design details of Secure Federated Submodel Learning (SFSL).

\subsection{Design Rationale}\label{sec:design:rationale}

We illustrate our key design principles mainly through demonstrating how to handle two fundamental problems raised in Section~\ref{sec:introduction:problems:challenges} and how to resolve several practical issues.


As shown in Fig.~\ref{fig:design:rationale}, we handle two fundamental problems in a unified manner rather than in separate ways. During both download and upload phases, a client consistently uses a perturbed index set in place of her real index set. In contrast, during the local training phase, the client leverages the intersection of her real index set and perturbed index set to prepare the succinct submodel and involved user data. With the blinding of the perturbed index set to interact with the outside world, the client holds a plausible deniability of some index being in or not in her real index set. Specifically, the client generates her perturbed index set locally with randomized response as follows. First, the sensitive question asked by the cloud server here is ``Do you have a certain index?". Then, the client answers ``Yes" with two customized probabilities, conditioned on whether the index is or is not in her real index set. These two probabilities allow the client to fine-tune balance between privacy and utility.

We further carefully examine the feasibility of our index set perturbation method in handling two fundamental problems. For the first problem in the download phase, if a client intends to download a certain row of the full matrix, and she actually downloads, she can blame her action to randomization, \ie, the occurrence of the event that the index not in a client's real index set returns a ``Yes" answer. Similarly, if a client does not intend to download the row, and she actually does not, she can still attribute her action to randomization, \ie, the occurrence of the event that the index in a client's real index returns a ``No" answer. Regarding the second problem in the upload phase, the usage of the perturbed index set still empowers a client to deny her underlying intention of modifying or not modifying some row of the full matrix, even if the cloud server observes her binary action of modifying or not modifying. Additionally, for a concrete row, there are two different groups of clients involved in the joint modification: (1) One group consists of those clients who intend to modify the row and contribute nonzero modifications; and (2) the other group comprises those clients who do not intend to modify the row and pretend to modify by submitting zero modifications. Under the secure aggregation guarantee, even though the cloud server observes the aggregate modification, it is hard for the cloud server, as an adversary, to identify any individual modification and further to infer whether some client originally intends to perform a modification or not. The hardness is controlled by the sizes of two groups, alternatively the probabilities of an index in and not in the real index set returning a ``Yes" answer, which are fully tunable by clients.


In addition to these two basic problems, there still exist two practical issues to be solved before the above index set perturbation method can apply to federated submodel learning. The first issue regards practical efficiency, \ie, whether it is practical and necessary for the cloud server to ask ``Do you have a certain index?" for each index in the full index set. In our context, the number of rows of the matrix, representing the deep learning model, is in the magnitude of billions. Thus, it is impractical for a client to answer billion-scale questions and further download and securely upload those rows with ``Yes" answers of the full model. We thus turn to narrowing down the size of the questions and identify a sufficient and necessary index set, namely the union of the chosen clients' real index sets. Our optimization is inspired by an example: if a client's real index set is of size 100, and the full index set is of size 1 billion, using the probability parameters in the survey of party membership, her expected number of ``Yes" answers is $10^2 \times 75\% + (10^{9} - 10^2) \times 25\% \approx 2.5 \times 10^8$. Such a calculation implies that the dominant ``Yes" answers are those with ``No" in reality but ``Yes" due to randomness. Nevertheless, most of the ``No"-to-``Yes" answers are useless. More specifically, for those indices that do not belong to any client's real index set, \eg, index 5 in Fig.~\ref{fig:adversary:model} and Fig.~\ref{fig:design:rationale}, although part ($25\%$ in expectation) of clients upload zero vectors for randomization, the cloud server can still infer from the aggregate zero vectors that these clients do not actually have the indices. Thus, it is not necessary to cover any index outside the union.

Accompanied with the first issue, another fundamental and thorny problem that arises is how multiple clients can obtain the union of their real index sets under the mediation of an untrusted cloud server without revealing any individual client's real index set to others, \ie, the need of a private set union protocol. Considering any existing scheme doest not satisfy the atypical setting and the stringent requirement of federated submodel learning, we design a novel private set union scheme based on Bloom filter, secure aggregation, and randomization. We first let each chosen client represent her real index set as a Bloom filter. Then, different from the common practice to derive the union of sets by performing bit-wise OR operations over their Bloom filters, which naturally requires both addition and multiplication operations, we let the cloud server directly ``sum" the Bloom filters. Here, the sum operation can be performed obliviously and efficiently under the coordination of an untrusted cloud server with the secure aggregation protocol. The aggregate Bloom filter is actually a counting Bloom filter, equivalent to constructing it from scratch by inserting each set in sequence. Besides the membership information, the counting Bloom filter still contains the count number of each element in the union. To prevent such an undesirable leakage, we let each client replace bit 1s in her Bloom filter with random integers, while keeping each bit 0 unchanged. When recovering the union of real index sets, one naive method for the cloud server is to do membership tests for the full index set, which is prohibitively time consuming, and can also introduce a large number of false positives. To handle these problems, we let the cloud server first divide the full index set into a certain number of partitions, and then let each client fill in a bit vector to indicate whether there exists an element of her real index set falling into these partitions. Then, just like computing the union with Bloom filters, the cloud server can securely determine those partitions that contain clients' real indices to further facilitate efficient union reconstruction.



The second issue regards the longitudinal privacy guarantee when a client is chosen to participate in multiple communication rounds. However, the initial version of randomized response only provides a rigorous privacy guarantee when an audience answers the same question once, facilitating only a one-time response to ``Do you have a certain index?" in our context. Thus, we need to extend the original randomized response mechanism to allow repeated responses from the same client to those already answered indices in a privacy-preserving manner. Our extension leverages key principles from Randomized Aggregatable Privacy-Preserving Ordinal Response (RAPPOR)~\cite{proc:ccs14:rappor,jour:popets:rappor2} and plays the randomized response game twice with a memoization step between. Specifically, the noisy answers generated by the inner randomized response will be memoized and permanently replace real answers in the outer randomized response. This ensures that even though a client responds to the membership of a concrete index for an infinite number of times, she can still hold a plausible deniability of her real answer, where the level of deniability is lower bounded by the memoized noisy answer.


\subsection{Design Details}

Following the guidelines in Section~\ref{sec:design:rationale}, we propose a secure scheme for federated submodel learning. We introduce the scheme in a top-down manner, where we first give an overview of its top-level architecture and then show two underlying modules, namely index set perturbation and private set union. For the sake of clarity, we outline our design in Algorithm~\ref{alg:fsl}, Algorithm~\ref{alg:index:set:ldp}, and Algorithm~\ref{alg:set:union}.

\subsubsection{Secure Federated Submodel Learning}

\begin{algorithm}[!t]
\SetNoFillComment
\small
\caption{Secure Federated Submodel Learning}\label{alg:fsl}
\tcc{Cloud server's process}
Initializes the global model $\mathbf{W}$\;
\ForEach{{\em communication round}}
{
   Randomly selects $n$ clients, denoted as $\mathcal{C}$, where $|\mathcal{C}| = n$\;
   Launches private set union (Algorithm~\ref{alg:set:union}), gets the union of $n$ clients' real index sets, namely $\bigcup_{i \in \mathcal{C}} \mathcal{S}^{(i)}$, and delivers the union result to the up-to-date set of clients who are alive, denoted as $\hat{\mathcal{C}} \subset \mathcal{C}$\;
   \ForEach{Client $i \in \hat{\mathcal{C}}$}
   {
      Receives and stores the perturbed index set $\mathcal{S}''^{(i)}$ from client $i$, and returns the submodel $\mathbf{W}_{\mathcal{S}''^{(i)}}$ and training hyperparameters to $i$\;
   }
   \tcp{Securely aggregates weighted submodel updates and count vectors}
   \ForEach{ $j \in \bigcup_{i \in \mathcal{C}} \mathcal{S}^{(i)}$}
   {
      Determines the live clients involving index $j$, denoted as $\hat{\mathcal{C}}_j = \{i \vert i \in \hat{\mathcal{C}} \wedge j \in \mathcal{S}''^{(i)}\}$, lets them submit materials for secure aggregation, and obtains the sum of (weighted) updates $\sum_{i \in \hat{\mathcal{C}}_j} \Delta \mathbf{w}_j^{(i)}$ and the total count number of relevant samples $\sum_{i \in \hat{\mathcal{C}}_j} v_j^{(i)}$\;
      Updates the $j$-th row of the global model $\mathbf{W}$ by adding $\left. \sum_{i \in \hat{\mathcal{C}}_j} \Delta \mathbf{w}_j^{(i)} \middle /\sum_{i \in \hat{\mathcal{C}}_j} v_j^{(i)} \right.$\;
   }
}
\tcc{Client $i$'s process}
Determines her real index set $\mathcal{S}^{(i)}$ based on local data\;
Participates in private set union (Algorithm~\ref{alg:set:union})\;
Generates a perturbed index set $\mathcal{S}''^{(i)}$ (Algorithm~\ref{alg:index:set:ldp})\;
Uses $\mathcal{S}''^{(i)}$ to download a submodel, denoted as $\mathbf{W}_{\mathcal{S}''^{(i)}}$\;
Depending on the succinct index set $\mathcal{S}''^{(i)} \bigcap \mathcal{S}^{(i)}$, locally extracts the succinct submodel $\mathbf{W}_{\mathcal{S}''^{(i)} \bigcap \mathcal{S}^{(i)}}$ from $\mathbf{W}_{\mathcal{S}''^{(i)}}$ and prepares involved data as the succinct training set\;
Locally trains $\mathbf{W}_{\mathcal{S}''^{(i)} \bigcap \mathcal{S}^{(i)}}$ using the hyperparameters and obtains the update of succinct submodel $\Delta \mathbf{W}_{\mathcal{S}''^{(i)} \bigcap \mathcal{S}^{(i)}}$\;
Initializes the submodel update to be uploaded, denoted as $\Delta \mathbf{W}_{\mathcal{S}''^{(i)}}$, all to 0, and then adds $\Delta \mathbf{W}_{\mathcal{S}''^{(i)} \bigcap \mathcal{S}^{(i)}}$\;
Counts the number of samples involving each index $j \in \mathcal{S}''^{(i)}$ and stores the results to the vector $\mathbf{v}_{\mathcal{S}''^{(i)}}$\;
Updates $\Delta \mathbf{W}_{\mathcal{S}''^{(i)}}$ by multiplying each row with the corresponding count number in $\mathbf{v}_{\mathcal{S}''^{(i)}}$\;
Uploads materials for securely aggregating $\Delta \mathbf{W}_{\mathcal{S}''^{(i)}}, \mathbf{v}_{\mathcal{S}''^{(i)}}$.
\end{algorithm}

Before presenting our secure federated submodel learning framework, we first briefly review the federated averaging algorithm~\cite{proc:aistats17:fedavg}, which is the cornerstone and core of conventional federated learning. In particular, federated averaging is a synchronous distributed learning method, for non-iid and unbalanced training data distributed at massive communication-constrained clients, under the coordination of a cloud server. At the beginning of one communication round, the cloud server sends the up-to-date parameters of the global model and the training hyperparameters to some clients. Here, the training hyperparameters include the optimization algorithm, typically mini-batch SGD, the local batch size (the number of training samples used to locally update the global model once, namely per iteration), the local epochs (the number of passes over a client's entire training data), and the learning rate. Then, each chosen client trains the global model on her data and uploads the update of the global model together with the size of her training data to the cloud server. The cloud server takes a weighted average of all updates, where one client's weight is proportional to the size of her local data, and finally adds the aggregate update to the global model.



We now present secure federated submodel learning in Algorithm~\ref{alg:fsl}, which generalizes the federated averaging algorithm to support effective and efficient submodel learning and preserves desirable security and privacy properties while incorporating the unstable and limited network connections of mobile devices. At the initial stage, the cloud server randomly initializes the global model (Line 1). For each communication round, the cloud server first selects some clients to participate (Line 3) and also maintains an up-to-date set of clients who are alive throughout the whole round. A chosen client determines her real index set based on her local data, which can specify the ``position" of her truly required submodel (Line 10). For example, if the visited goods IDs of a Taobao user include $\{1, 2, 4\}$, then she requires the first, second, and fourth rows of the embedding matrix for goods IDs, which further implies that her real index set should contain $\{1, 2, 4\}$. Then, the cloud server launches the private set union protocol to obtain the union of all chosen clients' real index sets while keeping each individual client's real index set in secret (Lines 4 and 11). The union result will be further delivered to live clients, based on which each client can perturb her real index set with a customized local differential privacy guarantee (Line 12). In addition, each client will use the perturbed index set, rather than the real index set, to download her submodel and upload the submodel update (Lines 13 and 19). In other words, when interacting with the cloud server, a client's real index set is replaced with her perturbed index set, which provides deniability of her real index set and thus obscures her training data. Upon receiving the perturbed index set from a client, the cloud server stores it for later usage and returns the corresponding submodel and the training hyperparameters to the client (Line 6). Depending on the intersection of the real index set and the perturbed index set, called the succinct index set, the client extracts a succinct submodel and prepares involved data as the succinct training set (Line 14). For example, if a Taobao user's real index set is $\{1, 2, 4\}$ and her perturbed index set is $\{2, 4, 6, 9\}$, she receives a submodel with row indices $\{2, 4, 6, 9\}$ from the cloud server, but just needs to train the succinct submodel with row indices $\{2, 4\}$ over her local data involving goods IDs $\{2, 4\}$. After training under the preset hyperparameters, the client derives the update of the succinct submodel (Line 15) and further prepares the submodel update to be uploaded with the perturbed index set by adding the update of the succinct submodel to the rows with the succinct indices and padding zero vectors to the other rows (Line 16). Additionally, to facilitate the cloud server in averaging submodel updates according to the sizes of relevant local training data, each client also needs to count the number of her samples involving every index in the perturbed index set (Line 17). In particular, the numbers of samples involving the indices outside the succinct index set are all zeros. Furthermore, each client prepares the submodel update to be uploaded by multiplying each row with the corresponding count number, namely the weight, in advance (Line 18). Finally, the weighted submodel updates and the count vectors from live clients are securely aggregated under the guidance of the cloud server (Lines 7--9 and 19). Specifically, the cloud server guides the secure aggregation by enumerating every index in the union of the chosen clients' real index sets. For each index, the cloud server first determines the set of live clients whose perturbed index sets contain this index and then lets these clients submit the materials for securely adding up the weighted updates and the count numbers with respect to this index (Line 8). The cloud server finally applies the update to the global model in this row by adding the quotient of the sum of the weighted updates and the total count number, namely the weighted average (Line 9). Considering that the weighted submodel updates and the count numbers are aggregated side by side, each client can augment the matrix, denoting her weighted submodel update, with the transposed count vector in the last column, when preparing materials for secure aggregation (Line 19). In addition, to reduce the interactions between the cloud server and a client, they can package all the materials supporting secure aggregation, rather than exchange the materials for one index each time (Lines 7--9), \ie, the cloud server executes Lines 7 and 8 for each live client $i \in \hat{\mathcal{C}}$ in parallel and then executes Line 9 for each index in the union $j \in \bigcup_{i \in \mathcal{C}} \mathcal{S}^{(i)}$.


\subsubsection{Index Set Perturbation}\label{sec:design:index:set:perturbation}

\begin{algorithm}[!t]
\small
\caption{Client $i$'s Index Set Perturbation}\label{alg:index:set:ldp}
\KwIn{Client $i$'s real index set $\mathcal{S}^{(i)}$, Union of the chosen clients' real index sets $\bigcup_{i\in\mathcal{C}} \mathcal{S}^{(i)}$, Client $i$'s memoized index set $\mathcal{Y}^{(i)}$ (\resp, $\mathcal{N}^{(i)}$) initialized to $\emptyset$ at very beginning with ``Yes" (\resp, ``No") permanent answers to the question ``Do you have a certain index?", Client $i$'s customized probability parameters $0 \leq p_1^{(i)}, p_2^{(i)}, p_3^{(i)}, p_4^{(i)} \leq 1$.}
\KwOut{Client $i$'s doubly perturbed index set $\mathcal{S}''^{(i)}$}
$\mathcal{S}'^{(i)} = \emptyset, \mathcal{S}''^{(i)} = \emptyset$\;
\tcp{Permanent randomized response}
\ForEach{$j \in \bigcup_{i \in \mathcal{C}} \mathcal{S}^{(i)} \bigwedge j \notin \mathcal{Y}^{(i)} \bigcup \mathcal{N}^{(i)}$}
{
   \If{$j \in \mathcal{S}^{(i)}$}
   {
      Adds $j$ to $\mathcal{S}'^{(i)}$ with probability $p_1^{(i)}$\;
   }
   \Else
   {
      Adds $j$ to $\mathcal{S}'^{(i)}$ with probability $p_2^{(i)}$\;
   }
   \tcp{Memoization of permanent responses}
   \If{$j \in \mathcal{S}'^{(i)}$}
   {
      $\mathcal{Y}^{(i)} = \mathcal{Y}^{(i)} \bigcup j$\;
   }
   \Else
   {
      $\mathcal{N}^{(i)} = \mathcal{N}^{(i)} \bigcup j$\;
   }
}
\tcp{Instantaneous randomized response}
\ForEach{$j \in \bigcup_{i \in \mathcal{C}} \mathcal{S}^{(i)}$}
{
   \If{$j \in \mathcal{Y}^{(i)}$}
   {
      Adds $j$ to $\mathcal{S}''^{(i)}$ with probability $p_3^{(i)}$\;
   }
   \Else
   {
      Adds $j$ to $\mathcal{S}''^{(i)}$ with probability $p_4^{(i)}$\;
   }
}
\Return{$\mathcal{S}''^{(i)}$}
\end{algorithm}

We now present how a client can generate a perturbed index set to download her submodel and to upload the update of the submodel, with a customized local differential privacy guarantee against the cloud server. Just like the exemplary question about party membership, the sensitive question here is ``Do you have a certain index?", asked by the cloud server. The clients participating in one round of federated submodel learning make up the population to be surveyed. Thus, the clients can use randomized response to answer ``Yes" or ``No", which provides measurable deniabilities of their true answers. However, as sketched in Section~\ref{sec:design:rationale}, several practical issues need to be resolved so that randomized response can truly apply here. In what follows, we elaborate on our solution details on these issues.


As shown in Algorithm~\ref{alg:index:set:ldp}, we view the union of the chosen clients' real index sets, rather than the full index set, as the scope of the cloud server's questionnaire (Input). We reason about necessity and sufficiency as follows. Without loss of generality, we consider client $i$ and the other chosen clients $\mathcal{C}\backslash \{i\}$. If any client in $\mathcal{C}\backslash \{i\}$ wants to obtain deniability of her real indices, she requires client $i$ to join as an audience to answer the questions about her real indices, which implies that client $i$ should know the union of the other chosen clients' real index sets. By incorporating client $i$'s own real index set, the questionnaire to client $i$ should contain the union of all chosen clients' real index sets. We further apply the above reasoning to all clients $\forall i \in \mathcal{C}$ and can derive that the global questionnaire should contain the union of all chosen clients' real index sets. Next, we illustrate whether the union suffices. We consider any index outside the union. Under the conventional federated learning framework, each client will upload a zero vector to the cloud server for this index. When the cloud server learns that the sum is a zero vector, she can infer that all chosen clients do not have this index. Please see index 5 in Fig.~\ref{fig:adversary:model:full} for an intuition. From this perspective, any index outside the union does not need to be preserved in federated submodel learning as well. Nevertheless, suppose an index outside the union is introduced by chance, \eg, due to a false positive of Bloom filter when reconstructing union in Algorithm~\ref{alg:set:union}. A nice phenomenon occurs. On one hand, the privacy of federated submodel learning can be enhanced in the sense that the cloud server can only ascertain that those clients who return ``Yes" answers do not really have this index, but cannot ascertain the states of the other clients due to plausible deniability. On the other hand, those clients with ``Yes" answers need to download the row with respect to this index and further to upload zero vectors through secure aggregation, which are useless and increase their overheads.

Given the questionnaire, client $i$ basically uses two probability parameters $p_1^{(i)}, p_2^{(i)}$ in randomized response to fine-tune the tension among effectiveness, efficiency, and privacy (Lines 3--6). In particular, $p_1^{(i)}$ denotes the probability that an index in client $i$'s real index set will return a ``Yes" answer and controls the factual size of a client's user data contributed to federated submodel learning. Thus, a larger $p_1^{(i)}$ implies better effectiveness in terms of convergency rate. In addition, $p_2^{(i)}$ denotes the probability that an index outside client $i$'s real index set will return a ``Yes" answer and determines the number of redundant rows to be downloaded and the number of padded zero vectors to be uploaded through the secure aggregation protocol. Hence, given a fixed $p_1^{(i)}$, a smaller $p_2^{(i)}$ indicates higher efficiency. Furthermore, $p_1^{(i)}, p_2^{(i)}$ jointly adjust the level of local differential privacy, where a pair of closer $p_1^{(i)}, p_2^{(i)}$ provides a better privacy guarantee. We examine three typical examples: (1) The randomized response in the party membership survey takes $p_1^{(i)} = 75\%$ and $p_2^{(i)} = 25\%$ for each respondent; (2) conventional federated learning essentially leverages the full index, takes $p_1^{(i)} = 1$ and $p_2^{(i)} = 1$ for each client, and offers the best privacy and effectiveness guarantees but the worst efficiency guarantee; and (3) federated submodel learning adopts $p_1^{(i)} = 1$ and $p_2^{(i)} = 0$ for each client and provides the best effectiveness and efficiency guarantees but the worst privacy guarantee.


Considering that client $i$ can be chosen to participate in multiple communication rounds and needs to repeatedly respond to some answered indices, we extend the basic randomized response mechanism to offer a rigorous privacy guarantee, also called {\em longitudinal privacy} in the literature~\cite{proc:ccs14:rappor,jour:popets:rappor2,proc:soda19:ldp:longitudinal}. We adopt a memoization technique from RAPPOR. The core idea of RAPPOR is to play the randomized response game twice with a memoization step between. The first perturbation step, called permanent randomized response, is used to create a noisy answer, which is memoized by the client and permanently reused in place of the real answer. The second perturbation step, called instantaneous randomized response, reports on the memoized answer over time, eventually completely revealing it. In other words, the privacy level, guaranteed by the underlying memoized answer in the permanent randomized response, imposes a lower bound on the privacy level, ensured by each instantaneous/revealed response. When the memoization technique is applied to federated submodel learning, we let client $i$ maintain two index sets with ``Yes" and ``No" answers in the permanent randomized response, respectively (Input). Here, the permanent randomized response mechanism is parameterized by two probabilities $p_1^{(i)}, p_2^{(i)}$ to tune privacy and utility (Lines 3--6), as illustrated in the preceding paragraph. In addition, given that one client can be grouped with distinct clients in different communication rounds while the union of real index sets varies from one round to another, the client needs to handle new indices. As a new index comes (Line 2), client $i$ generates a permanent noisy answer for it and further updates her two memoized sets (Lines 7--10). Moreover, client $i$ obtains her final perturbed index set by performing an instantaneous randomized response over the memoized answers to the union of real index sets in the current communication round (Lines 11--16). In particular, the instantaneous randomized response is parameterized with another two probabilities $p_3^{(i)}, p_4^{(i)}$ (Lines 13 and 15), similar to $p_1^{(i)}, p_2^{(i)}$ in the permanent randomized response. Now, these four probability parameters jointly support tuning the tension among privacy, effectiveness, and efficiency. More specifically, $p_5^{(i)} = p_1^{(i)}(p_3^{(i)} - p_4^{(i)}) + p_4^{(i)}$, denoting the probability that an index in client $i$'s real index set finally returns a ``Yes" answer, and $p_6^{(i)} = p_2^{(i)}(p_3^{(i)} - p_4^{(i)}) + p_4^{(i)}$, denoting the probability that an index not in client $i$'s real index set finally returns a ``Yes" answer, now play the same roles as $p_1^{(i)}$ and $p_2^{(i)}$, respectively. Detailed derivations of $p_5^{(i)}, p_6^{(i)}$ are deferred to Section~\ref{sec:security:analysis}.


Finally, we provide some comments on the above design. First, our design is different from conventional locally differentially private schemes (\eg, randomized response and RAPPOR), which require each participating user to choose the same probability parameters (\ie, $\forall i \in \mathcal{C}, p_1^{(i)} = p_1, p_2^{(i)} = p_2, p_3^{(i)} = p_3, p_4^{(i)} = p_4$), so that true statistics (\eg, heavy hitter, histogram, and frequency) can be well estimated using collected noisy answers, particularly after additional corrections. Such a requirement/assumption is no longer needed in our design because the cloud server, as the aggregator, performs aggregate statistics based on secure aggregation rather than over the noisy answers, \eg, counting how many samples from the chosen clients involve a certain index in total (Algorithm~\ref{alg:fsl}, Line 8). Therefore, as mentioned above, different clients can customize probability parameters to tune privacy and utility. Second, our index perturbation mechanism in Algorithm~\ref{alg:index:set:ldp} needs a prerequisite that the real index set of a client does not change when she participates in different communication rounds. Considering that the real index set corresponds to the client's local data, this prerequisite can be further converted to the invariance of the client's local data. One feasible way to meet this prerequisite is to introduce the concept of ``period" into federated submodel learning, \eg, one period can be set to one month. In a concrete period, a client uses the historical data in the previous one period to participate in federated submodel learning for several communication rounds. In addition, when entering a new period, the client just restarts Algorithm~\ref{alg:index:set:ldp}. The other feasible way is to allow changes in a client's real index set from one communication round to another. This implies that the underlying binary states of some indices may change. For example, if a client's local data and thus her real index set expand incrementally, some indices, which were not in, can fall into the real index set in later rounds. Recently, Erlingsson~\et~\cite{proc:soda19:ldp:longitudinal} considered a similar setting, in particular the collection of user statistics (\eg, software adoption) for multiple times with each user changing her underlying boolean value for a limited number of times. Therefore, their design, based on binary tree and Bernoulli distribution, can be leveraged to extend Algorithm~\ref{alg:index:set:ldp}, allowing a client to change her local data and thus her real index set in different communication rounds.

\subsubsection{Private Set Union}

\begin{algorithm}[t]
\small
\caption{Private Set Union}\label{alg:set:union}
\KwIn{Client $i$'s real index set $\mathcal{S}^{(i)}$ for all $i \in \mathcal{C}$}
\KwOut{Union of real index sets $\bigcup_{i\in\mathcal{C}} \mathcal{S}^{(i)}$}
Cloud server determines the partitions of the full index set $\mathcal{S}$\;
\ForEach{{\em Client} $i \in \mathcal{C}$}
{
  Represents $\mathcal{S}^{(i)}$ as a Bloom filter $\mathbf{b}^{(i)}$\;
  Perturbs $\mathbf{b}^{(i)}$ to an integer vector $\mathbf{b}'^{(i)}$ by replacing each bit 1 in $\mathbf{b}^{(i)}$ with a random integer from $\mathbb{Z}_R$\;
  Uses a bit vector $\mathbf{a}^{(i)}$ to indicate whether there exists an element in $\mathcal{S}^{(i)}$ falling into the partitions of $\mathcal{S}$\;
  Perturbs $\mathbf{a}^{(i)}$ to an integer vector $\mathbf{a}'^{(i)}$ by replacing each bit 1 in $\mathbf{a}^{(i)}$ with a random integer from $\mathbb{Z}_R$\;
  Submits materials for securely aggregating $\mathbf{b}'^{(i)}, \mathbf{a}'^{(i)}$\;
}
Cloud server obtains $\sum_{i \in \mathcal{C}} \mathbf{b}'^{(i)}$ and $\sum_{i \in \mathcal{C}} \mathbf{a}'^{(i)}$ with the secure aggregation protocol, reconstructs the union $\bigcup_{i\in\mathcal{C}} \mathcal{S}^{(i)}$, and delivers $\bigcup_{i\in\mathcal{C}} \mathcal{S}^{(i)}$ to each live client $i \in \hat{\mathcal{C}}$.
\end{algorithm}

We introduce the last module of our design: private set union. We first briefly review related work about private set operations, with a focus on the often overlooked but significantly important private set union. Then, we outline the practical infeasibility of existing protocols when adapted to federated submodel learning. We finally present our efficient and scalable scheme.

The goal of a private set operation protocol is to allow multiple parties, where each party holds a private set, to obtain the result of an operation over all the sets, without revealing each individual private set and without introducing a trusted third party. Compared with Private Set Intersection (PSI)~\cite{proc:eurocrypt04:psi,proc:ccs13:psi,proc:eurocrypt17:psi,proc:ccs17:psi:contact:discovery,proc:ccs17:psi:facebook,jour:tops18:psi} and Private Set union Cardinality (PSC)~\cite{proc:cans12:psc,proc:acisp15:psc,proc:ccs17:psc}, which have received tremendous attention and also have seen several practical applications, such as in social networks~\cite{proc:cans09:psi:soical,proc:acsac13:psi:friend}; human genome testing~\cite{proc:ccs11:psi:dna}; location-based services~\cite{proc:ndss11:psi:location}; security incident information sharing~\cite{proc:usenixs15:psi}; online advertising~\cite{proc:usenixs15:psi}; private contact discovery~\cite{proc:ccs17:psi:contact:discovery}; and the Tor anonymity network~\cite{proc:imc18:psc:tor}, there is little work and negligible focus on Private Set Union (PSU). Nevertheless, private set union promises a wide range of applications in practice, \eg, union queries over several databases, and, more generally, integration/sharing of datasets from multiple private sources. Thus, independent of federated submodel learning, the task of designing a practical private set union protocol itself is highly desired and urgent. Existing protocols mainly come from the fields of data mining and cryptography. In the data mining field, the representative design of private set union~\cite{jour:kddexp02:datamining:psu} is based on commutative encryption and requires direct communication between any pair of two parties. Unfortunately, the design leaks the cardinality of any two-party set intersection, and the underlying commutative encryption is fragile as well. In the cryptography field, according to the representation format of a set, the protocols can be generally divided into two categories: polynomial based~\cite{proc:crypto05:song:ppso,proc:acns07:frikken:psu,proc:pkc12:psu,jour:korea13:psu,tc:iacr19:psu} and Bloom filter based~\cite{tc:eth:pso:sepia,proc:acisp17:pso,proc:smartcomp18:miyaji}. For the polynomial-based protocols, elements of a set are represented as the roots of a polynomial, and the union of two sets is converted to the multiplication of two polynomials. For the protocols based on Bloom filter, the union operation over sets is normally transformed to the element-wise OR operation over Bloom filters, as demonstrated in Section~\ref{sec:bloom:filter}, whereas the logic OR operation can be further converted to bit addition and bit multiplication. To obliviously perform addition and multiplication operations, the above two kinds of protocols mainly turn to generic secure two-party/multiparty computation (\eg, garbled circuit, homomorphic encryption, secret sharing, and oblivious transfer), or outsource secure computation to multiple noncolluding servers. Due to unaffordable computation and communication overheads, none of the existing private set union protocols have been deployed in practice. In addition to inefficiency, the basic setting of these protocols significantly differs from that of federated submodel learning, where clients cannot directly communicate with each other and should mediate through an untrusted cloud server. Additionally, the set elements here can come from a billion-scale domain, which has not been touched in previous work as of yet.

Given the infeasibility of existing protocols and the atypical setting of federated submodel learning, we present our new private set union scheme in Algorithm~\ref{alg:set:union}. First, each client represents her real index set as a Bloom filter (Line 3). The details about how to set the parameters of the Bloom filter can be found in Section~\ref{sec:bloom:filter}. Second, different from the common practice to derive the union of sets by performing bitwise OR operations over their Bloom filters, which requires both addition and multiplication operations, we let the cloud server directly sum the Bloom filters. Here, the sum operation can be conducted obliviously and efficiently under the coordination of the untrusted cloud server with secure aggregation (Line 8). In addition, the resulting Bloom filter is actually a counting Bloom filter, equivalent to constructing it from scratch by sequentially inserting each real index set. In addition to membership information, the counting Bloom filter also contains the count numbers of elements in the union of real index sets. In other words, the cloud server may learn how many clients have a certain index, which is undesired in our context. To prevent the leakage of count numbers, we let each client generate a perturbed integer vector, which replaces each bit 1 in her Bloom filter with a random integer and keeps each bit 0 unchanged (Line 4). Such a perturbation process can obscure count numbers while retaining membership information. Third, after obtaining the sum of perturbed Bloom filters, the cloud server can recover the union of real index sets by doing membership tests for the full index set. For example, to judge whether an index belongs to the union, we check the resulting integer vector at the positions of its hash values. The index is considered to be in the union only if all the values are nonzero. However, one practical issue arises: the domain of index can be huge, \eg, the full size of the goods IDs in Taobao is in the magnitude of billions. Thus, it can be prohibitively time consuming to enumerate all indices. Even worse, the direct enumeration method can also introduce a large number of false positives in the union, \ie, those indices not in the union are falsely judged to be in, which can further lead to unnecessary redundancy in the download and upload phases. To handle these problems, we further incorporate a private ``partition" union to narrow down the scope of index for union reconstruction above. We let the cloud server divide the full domain of the index into a certain number of partitions ahead of time (Line 1). A good partition scheme needs to well balance the pros in the union reconstruction phase and the cons of additional cost. Given the partitions, each client first uses a bit vector to record whether there exists an index in her real index set falling into the partitions (Line 5). Just the same as the Bloom filter to hide the concrete count numbers, the client further replaces each bit 1 with a random integer (Line 6). Then, the cloud server obtains the sum of the integer vectors using the secure aggregation protocol and reveals those partitions with nonzero integers in the corresponding positions. By simply doing membership tests for the indices falling into these partitions, the cloud server can efficiently construct the union. Last, the union is delivered to all live clients (Line 8).

\section{Security and Performance Analyses}\label{sec:security:performance:analysis}

In this section, we first analyze the privacy guarantees of our secure federated submodel learning scheme according to Definition~\ref{def:ldp} and Definition~\ref{def:client:control:privacy}, \ie, Theorem~\ref{theorem:alg:index:ldp:epsilon} and Theorem~\ref{theorem:client:control:2probabilities}. We also provide an instantiation of our scheme, where each client consistently uses the union of the chosen clients' real index sets when interacting with the cloud server, and prove that its security and privacy guarantees are as strong as conventional federated learning with secure aggregation (hereinafter also called ``Secure Federated Learning" and abbreviated as ``SFL"), \ie, Theorem~\ref{theorem:same:security}. We then show the proven security of our proposed private set union protocol, \ie, Theorem~\ref{theorem:set:union}. We finally analyze the performance of our scheme by comparing with that of secure federated learning.


\subsection{Security and Privacy Analyses}\label{sec:security:analysis}


By Definition~\ref{def:ldp}, we analyze the local differential privacy guarantee of index set perturbation in Algorithm~\ref{alg:index:set:ldp}. As stepping stones, we first analyze permanent randomized response and a one-time instantaneous randomized response in Lemma~\ref{lem:permanent:rp} and Lemma~\ref{lem:instantaneous:dp}, which impose an upper bound and a lower bound on the privacy level of Algorithm~\ref{alg:index:set:ldp}, namely Theorem~\ref{theorem:alg:index:ldp:epsilon}.

\begin{lemma}\label{lem:permanent:rp}
Permanent randomized response in Algorithm~\ref{alg:index:set:ldp} for client $i$ achieves local differential privacy at the level $\epsilon_{\infty}^{(i)} = \ln\left(\max\left(\frac{p_1^{(i)}}{p_2^{(i)}}, \frac{p_2^{(i)}}{p_1^{(i)}}, \frac{1 - p_1^{(i)}}{1 - p_2^{(i)}}, \frac{1 - p_2^{(i)}}{1 - p_1^{(i)}}\right)\right)$.
\end{lemma}
\begin{proof}
We focus on a certain index $j \in \bigcup_{i \in \mathcal{C}} \mathcal{S}^{(i)}$. According to Definition~\ref{def:ldp}, we need to consider all possible pairs of inputs from client $i$ and all possible outputs of the permanent randomized response in Algorithm~\ref{alg:index:set:ldp}. Here, the input pair is $j$ in and not in client $i$'s real index set, namely $j \in \mathcal{S}^{(i)}$ and $j \notin \mathcal{S}^{(i)}$. In addition, the possible outputs are $j$ obtaining ``Yes" and ``No" noisy answers for memoization, namely $j \in \mathcal{Y}^{(i)}$ and $j \in \mathcal{N}^{(i)}$. We thus can compute four ratios between the conditional probabilities of a permanent output with a pair of distinct inputs: $\frac{\mathrm{Pr}\left(j \in \mathcal{Y}^{(i)} \vert j \in \mathcal{S}^{(i)} \right)}{\mathrm{Pr}\left(j \in \mathcal{Y}^{(i)} \vert j \notin \mathcal{S}^{(i)} \right)} = \frac{p_1^{(i)}}{p_2^{(i)}}$, $\frac{\mathrm{Pr}\left(j \in \mathcal{Y}^{(i)} \vert j \notin \mathcal{S}^{(i)} \right)}{\mathrm{Pr}\left(j \in \mathcal{Y}^{(i)} \vert j \in \mathcal{S}^{(i)} \right)} = \frac{p_2^{(i)}}{p_1^{(i)}}$, $\frac{\mathrm{Pr}\left(j \in \mathcal{N}^{(i)} \vert j \in \mathcal{S}^{(i)} \right)}{\mathrm{Pr}\left(j \in \mathcal{N}^{(i)} \vert j \notin \mathcal{S}^{(i)} \right)} = \frac{1 - p_1^{(i)}}{1 - p_2^{(i)}}$, and $\frac{\mathrm{Pr}\left(j \in \mathcal{N}^{(i)} \vert j \notin \mathcal{S}^{(i)} \right)}{\mathrm{Pr}\left(j \in \mathcal{N}^{(i)} \vert j \in \mathcal{S}^{(i)} \right)} = \frac{1 - p_2^{(i)}}{1 - p_1^{(i)}}$.
By Definition~\ref{def:ldp}, we can derive the level of local differential privacy $\epsilon_{\infty}^{(i)}$: $\exp\left(\epsilon_{\infty}^{(i)}\right) = \max\left(\frac{p_1^{(i)}}{p_2^{(i)}}, \frac{p_2^{(i)}}{p_1^{(i)}}, \frac{1 - p_1^{(i)}}{1 - p_2^{(i)}}, \frac{1 - p_2^{(i)}}{1 - p_1^{(i)}}\right) \Rightarrow \epsilon_{\infty}^{(i)} = \ln\left(\max\left(\frac{p_1^{(i)}}{p_2^{(i)}}, \frac{p_2^{(i)}}{p_1^{(i)}}, \frac{1 - p_1^{(i)}}{1 - p_2^{(i)}}, \frac{1 - p_2^{(i)}}{1 - p_1^{(i)}}\right)\right)$.
\end{proof}

\begin{lemma}\label{lem:instantaneous:dp}
A one-time instantaneous randomized response in Algorithm~\ref{alg:index:set:ldp} for client $i$ satisfies local differential privacy at the level $\epsilon_1^{(i)}$, where $\epsilon_1^{(i)} = \ln\left(\max\left(\frac{p_5^{(i)}}{p_6^{(i)}}, \frac{p_6^{(i)}}{p_5^{(i)}}, \frac{1 - p_5^{(i)}}{1 - p_6^{(i)}}, \frac{1 - p_6^{(i)}}{1 - p_5^{(i)}}\right)\right)$, $p_5^{(i)} = \mathrm{Pr}\left(j \in \mathcal{S}''^{(i)} \vert j \in \mathcal{S}^{(i)}\right) = p_1^{(i)}\left(p_3^{(i)} - p_4^{(i)}\right) + p_4^{(i)}$, and $p_6^{(i)} = \mathrm{Pr}\left(j \in \mathcal{S}''^{(i)} \vert j \notin \mathcal{S}^{(i)}\right) = p_2^{(i)}\left(p_3^{(i)} - p_4^{(i)}\right) + p_4^{(i)}$.
\end{lemma}
\begin{proof}
The proof is similar to that of Lemma~\ref{lem:permanent:rp}. The difference is that the possible outputs are index $j$ being in and not in the final perturbed index set, namely $j \in \mathcal{S}''^{(i)}$ and $j \notin \mathcal{S}''^{(i)}$. We first compute two conditional probabilities $p_5^{(i)}$ and $p_6^{(i)}$, denoting the probabilities of $j$ in the final perturbed index set given an index $j$ is and is not in client $i$'s real index set, respectively. In particular, we can derive $p_5^{(i)}$ through
\begin{align}
\nonumber
p_5^{(i)} &=  \mathrm{Pr}\left(j \in \mathcal{S}''^{(i)} \middle\vert j \in \mathcal{S}^{(i)} \right)\\
\nonumber
    &= \mathrm{Pr}\left(j \in \mathcal{S}''^{(i)} \middle\vert j \in \mathcal{S}^{(i)}, j \in \mathcal{Y}^{(i)}\right) \mathrm{Pr}\left(j \in \mathcal{Y}^{(i)} \middle\vert j \in \mathcal{S}^{(i)} \right)\\
    &+ \mathrm{Pr}\left(j \in \mathcal{S}''^{(i)} \middle\vert j \in \mathcal{S}^{(i)}, j \in \mathcal{N}^{(i)}\right) \mathrm{Pr}\left(j \in \mathcal{N}^{(i)} \middle\vert j \in \mathcal{S}^{(i)} \right) \label{eq:lemma2:1} \\
\nonumber
    &= \mathrm{Pr}\left(j \in \mathcal{S}''^{(i)} \middle\vert j \in \mathcal{Y}^{(i)}\right) \mathrm{Pr}\left(j \in \mathcal{Y}^{(i)} \middle\vert j \in \mathcal{S}^{(i)} \right)\\
    &+ \mathrm{Pr}\left(j \in \mathcal{S}''^{(i)} \middle\vert j \in \mathcal{N}^{(i)}\right) \mathrm{Pr}\left(j \in \mathcal{N}^{(i)} \middle\vert j \in \mathcal{S}^{(i)} \right) \label{eq:lemma2:2} \\
\nonumber
    &= p_3^{(i)} p_1^{(i)} + p_4^{(i)}\left(1 - p_1^{(i)}\right)\\
\nonumber
    &= p_1^{(i)} \left(p_3^{(i)} - p_4^{(i)}\right) + p_4^{(i)},
\end{align}
where Equation~(\ref{eq:lemma2:1}) follows from the law of total probability, and Equation~(\ref{eq:lemma2:2}) follows that $j \in \mathcal{S}''^{(i)}$ is independent of $j \in \mathcal{S}^{(i)}$ conditioned on $j \in \mathcal{Y}^{(i)}$ or $j \in \mathcal{N}^{(i)}$. In a similar way, we can get $p_6^{(i)} = \mathrm{Pr}\left(j \in \mathcal{S}''^{(i)} \middle\vert j \notin \mathcal{S}^{(i)}\right) = p_2^{(i)} \left(p_3^{(i)} - p_4^{(i)}\right) + p_4^{(i)}.$ Based on $p_5^{(i)}$ and $p_6^{(i)}$, we can still compute four ratios between the conditional probabilities of an instantaneous output given a pair of different inputs and draw the level of local differential privacy $\epsilon_1^{(i)} = \ln\left(\max\left(\frac{p_5^{(i)}}{p_6^{(i)}}, \frac{p_6^{(i)}}{p_5^{(i)}}, \frac{1 - p_5^{(i)}}{1 - p_6^{(i)}}, \frac{1 - p_6^{(i)}}{1 - p_5^{(i)}}\right)\right)$.
\end{proof}

From the above deduction, we can draw that $p_5^{(i)}$ and $p_6^{(i)}$ in the instantaneous randomized response play the same roles as $p_1^{(i)}$ and $p_2^{(i)}$ in the permanent randomized response. This intuition has been given in Section~\ref{sec:design:index:set:perturbation}, and is now formally verified here.

By combining the above two lemmas, we show the level of local differential privacy ensured by Algorithm~\ref{alg:index:set:ldp}.



\begin{theorem}\label{theorem:alg:index:ldp:epsilon}
When client $i$ is chosen to participate in an arbitrary number of communication rounds, Algorithm~\ref{alg:index:set:ldp} satisfies $\epsilon^{(i)}$-local differential privacy, where $\epsilon_1^{(i)} \leq \epsilon^{(i)} \leq \epsilon_{\infty}^{(i)}$.
\end{theorem}
\begin{proof}
We consider that client $i$ participates in $k$ communication rounds of federated submodel learning, and Algorithm~\ref{alg:index:set:ldp} guarantees $\epsilon_k^{(i)}$-local differential privacy. Thus, client $i$ should generate $k$ instantaneous randomized responses. On one hand, suppose that an attacker only leverages the $k$-th instantaneous randomized response while ignoring all previous $k - 1$ instantaneous randomized responses. This corresponds to the strongest possible local differential privacy guarantee, namely the lower bound on $\epsilon_k^{(i)}$. According to Lemma~\ref{lem:instantaneous:dp}, a one-time instantaneous randomized response guarantees $\epsilon_1^{(i)}$-local differential privacy. Therefore, $\epsilon_1^{(i)} \leq \epsilon_k^{(i)}$. On the other hand, if the attacker leverages all $k$ instantaneous randomized responses, and as $k$ approaches positive infinity, the worst case is that the attacker reveals the permanent randomized response. This corresponds to the weakest possible local differential privacy guarantee, namely the upper bound on $\epsilon_k^{(i)}$. By Lemma~\ref{lem:permanent:rp}, the permanent randomized response can guarantee $\epsilon_{\infty}^{(i)}$-local differential privacy. Hence, $\epsilon_k^{(i)} \leq \epsilon_{\infty}^{(i)}$. We complete the proof.
\end{proof}

In fact, to derive the detailed local differential privacy guarantee when client $i$ participates in $k$ communication rounds, namely $\epsilon_k^{(i)}$, we need to make an additional assumption on how effectively the attacker infers client $i$'s permanent randomized response from all $k$ instantaneous randomized responses. We defer this explorative problem as our future work. Additionally, if we set $p_1^{(i)} = p_2^{(i)} = p_3^{(i)} = p_4^{(i)} = 1$, implying $p_5^{(i)} = p_6^{(i)} = 1, \epsilon_1^{(i)} = 0, \epsilon_\infty^{(i)} = 0,$ and $\epsilon^{(i)} = 0$, this corresponds to that any index, no matter whether it is or is not in client $i$'s real index set, will receive a permanent ``Yes" answer and an instantaneous ``Yes" answer. In other words, if client $i$ takes the union of the chosen clients' real index sets as her perturbed index set, the local differential privacy guarantee of Algorithm~\ref{alg:index:set:ldp} is $0$, which is the strongest case, as in conventional federated learning.

We next analyze our secure submodel updates aggregation in Algorithm~\ref{alg:fsl}, according to Definition~\ref{def:client:control:privacy}.

\begin{theorem}\label{theorem:client:control:2probabilities}
Algorithm~\ref{alg:fsl} is a client-controllable privacy protection mechanism for submodel updates aggregation. In particular, for any index $j \in \bigcup_{i \in \mathcal{C}} \mathcal{S}^{(i)}$, we let $n_{j, 0}$ and $n_{j, 1}$ denote the numbers of live clients who do not have and have $j$ in reality, respectively. If each live client chooses the same probability parameters (\oie, $\forall i \in \hat{\mathcal{C}}, p_1^{(i)} = p_1, p_2^{(i)} = p_2, p_3^{(i)} = p_3, p_4^{(i)} = p_4, p_5 = p_1(p_3 - p_4) + p_4$, and $p_6 = p_2(p_3 - p_4) + p_4$), then Algorithm~\ref{alg:fsl} can guarantee: Event 1 happens with probability $p_{7} = p_5 (1 - p_5)^{n_{j, 1} - 1} (1 - p_6)^{n_{j, 0}}$, and Event 2 happens with probability $p_{8} = (1 - p_5)^{n_{j, 1}}(1 - (1 - p_6)^{n_{j,0}})$.
\end{theorem}
\begin{proof}
Event 1 happens when only one of the $n_{j, 1}$ clients who have $j$ in reality submits a nonzero update while all $n_{j, 0}$ clients who do not have $j$ in reality also do not submit zero updates. By the product rule, we can compute the joint probability of Event 1 as $p_{7} = p_5 (1 - p_5)^{n_{j, 1} - 1} (1 - p_6)^{n_{j, 0}}$.

Event 2 happens when all $n_{j, 1}$ clients who have $j$ in reality do not submit nonzero updates. Under such a circumstance, if part (at least one) of $n_{j, 0}$ clients who do not have $j$ in reality submit zero updates, from the aggregate zero update, the cloud server almost ascertains that these clients do not have $j$ in reality. According to the product rule, we can compute the probability of Event 2 as $p_{8} = (1 - p_5)^{n_{j, 1}}(1 - (1 - p_6)^{n_{j, 0}})$.
\end{proof}

Theorem~\ref{theorem:client:control:2probabilities} enables participating clients to jointly adjust the privacy level of secure submodel updates aggregation by choosing different probability parameters. Details on fine-tuning are deferred to Appendix~\ref{app:tune:submodel:aggregation}. Additionally, we still examine the case that each client uses the union of the chosen clients' real index sets to upload the submodel update by setting $p_1 = p_2 = p_3 = p_4 = 1$, implying $p_5 = p_6 = 1$ and $p_{7} = p_{8} = 0$. This is the strongest possible client-controllable privacy for secure submodel updates aggregation, as in secure federated learning. Combining with the local differential privacy guarantee, we can see that secure federated submodel learning with the setting $p_1 = p_2 = p_3 = p_4 = 1$ holds the same security and privacy levels as secure federated learning under both Definition~\ref{def:ldp} and Definition~\ref{def:client:control:privacy}. We further generalize this observation in Theorem~\ref{theorem:same:security}, which is free of any security or privacy definition.


\begin{theorem}\label{theorem:same:security}
If the probability parameters $p_1^{(i)}, p_2^{(i)}, p_3^{(i)}, p_4^{(i)}$ all take 1s for each chosen client $i \in \mathcal{C}$, the security and privacy of secure federated submodel learning scheme in Algorithm~\ref{alg:fsl} are as strong as secure federated learning.
\end{theorem}
\begin{proof}
We consider any index $j$ from the full index set $\mathcal{S}$ (\ie, $j \in \mathcal{S}$), and prove in two complementary cases as follows.

Case 1 ($j \in \bigcup_{i \in\mathcal{C}} \mathcal{S}^{(i)}$): In both secure federated submodel learning and secure federated learning, each client $i \in \mathcal{C}$ will download $j$-th row of the full model and then upload the update of this row to the cloud server through secure aggregation. Specifically, if $j$ is not in a client's real index set, then this client will upload a zero vector. The whole processes of two schemes are consistent, implying the same security and privacy guarantees.

Case 2 ($j \notin \bigcup_{i \in\mathcal{C}} \mathcal{S}^{(i)}$): In secure federated submodel learning, each client will not download $j$-th row and also will not upload the zero vector. Thus, the adversary knows that each client doesn't not have index $j$ (\ie, $\forall i \in \mathcal{C}, j \notin \mathcal{S}^{(i)}$), and each client's update is a zero vector. In contrast, in secure federated learning, each client will download $j$-th row and update a zero vector as her update. From the result that the aggregate update is a zero vector, the adversary still ascertains that each client does not have $j$ and her update is a zero vector, which are the same as in secure federated submodel learning.

By summarizing two cases, we complete the proof.
\end{proof}


We finally analyze the security of private set union. As a springboard, we give Lemma~\ref{lem:sum:random}, which states that the modular addition of one or more random integers from $\mathbb{Z}_R = \{0, 1, \ldots, R - 1\}$ is still uniformly random in $\mathbb{Z}_R$. We note that the elementary operation in secure aggregation~\cite{proc:ccs17:secure:agg} is modular addition with a large modulus rather than original addition, which is consistent with Lemma~\ref{lem:sum:random}. In addition, in the field of number theory, the set of integers $\mathbb{Z}_R = \{0, 1, 2, \ldots, R - 1\}$ is called the least residue system modulo $R$, or the ring of the integers modulo $R$. Moreover, the set $\mathbb{Z}_R$ together with the operation of modular addition form a finite cyclic group.

\begin{lemma}\label{lem:sum:random}
For any nonempty set $\mathcal{C}_1 \neq \emptyset$ and for all $i \in \mathcal{C}_1$, $r_i$ is randomly and independently chosen from $\mathbb{Z}_R = \{0, 1, \ldots, R - 1\}$, denoted as $r_i \in_R \mathbb{Z}_R$, then $\sum_{i \in \mathcal{C}_1} r_i \mod R$ is still uniformly random in $\mathbb{Z}_R$.
\end{lemma}
\begin{proof}
We prove by induction on the cardinality of $\mathcal{C}_1$, denoted as $\vert \mathcal{C}_1\vert$, where $\vert \mathcal{C}_1 \vert \geq 1$.

The base case is to show that the statement holds for $\vert \mathcal{C}_1 \vert = 1$. We let $\{r_a\}$ denote the element, where $r_a \in_R \mathbb{Z}_R$. Thus, it is trivial that $\sum_{i \in \mathcal{C}_1} r_i \mod R = r_a \in_R \mathbb{Z}_R$.

The inductive step is to prove that if the statement for any nonempty $\bar{\mathcal{C}}_1 \subset \mathcal{C}_1$ holds, then the statement for $\bar{\mathcal{C}}_1 \bigcup \{b\}$ where $b \notin \bar{{\mathcal{C}}}_1, b \in \mathcal{C}_1$ still holds. We let $r_a$ denote $\sum_{i \in \bar{\mathcal{C}}_1} r_i \mod R$. Hence, it suffices to show that $r_a \in_R \mathbb{Z}_R, r_b \in_R \mathbb{Z}_R \Rightarrow (r_a + r_b) \mod R \in_R \mathbb{Z}_R$. We prove this statement by showing $\mathrm{Pr}((r_a + r_b) = r \mod R ) = 1/R$ for any $r \in \mathbb{Z}_R$. The detailed deduction is shown as follows:
\begin{align}
\nonumber
  &\mathrm{Pr}\left( \left(r_a + r_b\right) = r \mod R  \right)\\
\nonumber
=\ &\sum_{k = 0}^{R - 1} \mathrm{Pr}\left(r_a = k, r_b = r - k \mod R\right)\\
=\ &\sum_{k = 0}^{R - 1} \mathrm{Pr}\left(r_a = k\right)\mathrm{Pr}\left(r_b = r - k \mod R\right)\label{eq:lem:uniform:random:1}\\
\nonumber
=\ &\sum_{k = 0}^{R - 1} \frac{1}{R}\frac{1}{R}\\
\nonumber
=\ &\frac{1}{R},
\end{align}
where Equation~(\ref{eq:lem:uniform:random:1}) follows that $r_a, r_b$ are independent.

By mathematical induction, we complete the proof.
\end{proof}

\begin{theorem}\label{theorem:set:union}
The private set union protocol in Algorithm~\ref{alg:set:union} is proven secure in the sense that only the union of the chosen clients' real index sets is revealed.
\end{theorem}
\begin{proof}
We recall that in Algorithm~\ref{alg:set:union}, a client first represents her real index set as a Bloom filter $\mathbf{b}^{(i)}$, then replaces bit 1s with random integers, denoted as $\mathbf{b}'^{(i)}$, and finally executes secure aggregation with other clients under the coordination of the cloud server. Additionally, just in the same manner as computing the union, the client joins in the survey of whether there exists an index in her real index set falling into predivided partitions. Hence, for the sake of conciseness, we here only elaborate on secure union computation.

First, according to the security analysis in~\cite{proc:ccs17:secure:agg}, the secure aggregation protocol is proven secure in both honest-but-curious and active adversary settings, where the adversaries can be the cloud server and participating clients. In more detail, from security and robustness, the secure aggregation protocol can guarantee that nothing but the aggregate result is revealed to the cloud server and all participating clients  even if part of clients may drop out during the aggregation process. When the secure aggregation guarantee applies to our context, only $\sum_{i \in \mathcal{C}} \mathbf{b}'^{(i)}$ is revealed, while any individual $\mathbf{b}'^{(i)}$ for $i \in \mathcal{C}$ is concealed from both the cloud server and the other chosen clients $\mathcal{C}\backslash \{i\}$. Given that $\mathbf{b}'^{(i)}$ is a postprocessing of $\mathbf{b}^{(i)}$, the underlying Bloom filter $\mathbf{b}^{(i)}$ and thus each client $i$'s real index set $\mathcal{S}^{(i)}$ are obscured.

Second, we prove that the revealed sum $\sum_{i \in \mathcal{C}} \mathbf{b}'^{(i)}$ contains no information other than the union $\bigvee_{i \in \mathcal{C}} \mathbf{b}^{(i)}$ from the view of adversary with any prior knowledge. Considering both $\bigvee_{i \in \mathcal{C}} \mathbf{b}^{(i)}$ and $\sum_{i \in \mathcal{C}} \mathbf{b}'^{(i)}$ are element-wise operation, we thus just need to focus on one-dimensional cases of $\mathbf{b}^{(i)}$ and $\mathbf{b}'^{(i)}$, \ie,  $\mathbf{b}^{(i)}$ degenerates to one bit $b^{(i)} \in \{0, 1\}$ and $\mathbf{b}'^{(i)}$ degenerates to one random integer $b'^{(i)} \in \mathbb{Z}_R$. We now consider two complementary cases of $\{b^{(i)} | i \in \mathcal{C}\}$ as follows.

Case 1 ($\forall i \in \mathcal{C}, b^{(i)} = 0$): The union $\bigvee_{i \in \mathcal{C}} b^{(i)} = 0$ is the same as the $\sum_{i \in \mathcal{C}} b'^{(i)} = 0$. This implies nothing but the union is leaked from the sum in Case 1.

Case 2 ($\exists i \in \mathcal{C}, b^{(i)} = 1$): The union is $\bigvee_{i \in \mathcal{C}} b^{(i)} = 1$. We compute the sum $\sum_{i \in \mathcal{C}} b'^{(i)}$ by splitting $\mathcal{C}$ into two parts: $\mathcal{C}_0 = \{i \vert i \in \mathcal{C}, b^{(i)} = 0\}$ and $\mathcal{C}_1 = \{i \vert i \in \mathcal{C}, b^{(i)}=1\}$. Then, we can derive
\begin{align}
\nonumber
\sum_{i \in \mathcal{C}} b'^{(i)} &= \sum_{i \in \mathcal{C}_0} b'^{(i)} + \sum_{i \in \mathcal{C}_1} b'^{(i)}\\
\nonumber
                                  &= 0 + \sum_{i \in \mathcal{C}_1} b'^{(i)}\\
                                  &\in_R \mathbb{Z}_R.\label{eq:theorem:union:1}
\end{align}
Here, Equation~(\ref{eq:theorem:union:1}) holds as follows. According to the antecedent $\exists i \in \mathcal{C}, b^{(i)} = 1$, we have $\mathcal{C}_1 \neq \emptyset$. Additionally, in Algorithm~\ref{alg:set:union} (Line 4), if $b^{(i)} = 1$, we have $b'^{(i)} \in_R \mathbb{Z}_R$. Now, by using Lemma~\ref{lem:sum:random}, we have $\sum_{i \in \mathcal{C}_1} b'^{(i)} \in_R \mathbb{Z}_R$. This indicates that the sum is just a uniformly random integer to the adversary's view. Furthermore, in the context of union, (1) this random integer takes positive values with probability $1 - 1/R$ and is further decoded as element ``1"; (2) the random integer takes value 0 with a negligible probability $1/R$ and further is wrongly\footnote{Even if each bit ``1" is replaced with a positive rather than non-negative integer, the sum under modular addition can still take value 0 with a negligible probability. However, the sum can be no longer uniformly random.} decoded as element ``0", \ie, a false negative happens with a negligible probability $1/R$. Hence, the sum reveals nothing but the union in Case 2.

By incorporating no leakage of any individual real index set and no leakage but the union from the aggregate result, we complete the proof.
\end{proof}

We give more interpretations about the word ``only" in Theorem~\ref{theorem:set:union}. With no exception, the count number of each element in the union (\ie, the cardinality of $\mathcal{C}_1$ in the proof) is also hidden. The reason is that for any non-empty $\mathcal{C}_1$, the sum is uniformly random (\ie, Equation~(\ref{eq:theorem:union:1}) holds). This further indicates that from the sum, all possible $\mathcal{C}_1$'s are computationally indistinguishable to the adversary. In other words, the adversary learns the exact cardinality with probability $1/|\mathcal{C}|$, which is the same as the probability of a random guess among all possible cardinalities $\{1, 2, \ldots, |\mathcal{C}|\}$.

\subsection{Performance Analysis and Comparison}\label{sec:performance:analysis}

\begin{table}[!t]
\caption{Complexities of Secure Federated Learning (SFL) and Secure Federated Submodel Learning (SFSL) at the same security and privacy levels as well as Private Set Union (PSU) in SFSL.}
\label{tab:fsl:fl:overhead}
\centering
\resizebox{\columnwidth}{!}{
\begin{tabular}[t]{ll|ccc}
\toprule
                        &     & Communication  & Computation    & Storage \tabularnewline
\midrule\midrule
\multirow{3}{*}{Client} & SFL  & $O(n + md)$    & $O(n^2 + nmd)$ & $O(n + md)$      \tabularnewline
                        & SFSL & $O(nsd)$       & $O(n^2sd)$     & $O(nsd)$         \tabularnewline
                        & PSU  & $O(ns)$        & $O(n^2s)$      & $O(ns)$          \tabularnewline

\midrule
\multirow{3}{*}{Server} & SFL  & $O(n^2 + nmd)$ & $O(n^2 md)$    & $O(n^2 + md)$    \tabularnewline
                        & SFSL & $O(n^2sd)$     & $O(n^3sd)$     & $O(n^2 + nsd)$   \tabularnewline
                        & PSU  & $O(n^2s)$      & $O(n^3s)$      & $O(n^2 + ns)$    \tabularnewline
\bottomrule
\multicolumn{5}{l}{*$\vert\bigcup_{i \in \mathcal{C}} \mathcal{S}^{(i)}\vert \ll \vert\mathcal{S}\vert \Rightarrow ns \ll m$.}
\end{tabular}
}
\end{table}

We now analyze the performance of our proposed secure federated submodel learning scheme. We first analyze the communication, computation, and storage (including both memory and disk loads) complexities of the client and the cloud server. Then, we introduce secure federated learning as a benchmark for comparison. For clarity, Table~\ref{tab:fsl:fl:overhead} summarizes the complexities of two secure schemes at the same levels of security and privacy along with the complexities of our proposed private set union protocol.

\subsubsection{Performance of Secure Federated Submodel Learning}


We focus on a certain communication round. For a concrete phase within the round, \eg, private set union or secure submodel updates and count vectors aggregations, we consider that all $n$ chosen clients are alive at the beginning, but may drop out during the process, which imposes upper bounds on the overheads of the phase. In addition, for feasibility and clarity in analysis, we let each client choose the same probability parameters $p_1, p_2, p_3, p_4$, resulting in the same $p_5, p_6$. Moreover, we assume that the expected cardinality of each client's real index set is $s$, which indicates that the expected cardinality of the union of $n$ chosen clients' real index sets $\bigcup_{i \in \mathcal{C}}\mathcal{S}^{(i)}$ is upper bounded by $ns$. Here, $ns$ is much less than the cardinality of the full index set $m$. Furthermore, the expected cardinality of each client's perturbed index set $\mathcal{S}''^{(i)}$ is upper bounded by $sp_5 + (n - 1)s p_6$, and the expected cardinality of each client's succinct index set $\mathcal{S}^{(i)}\bigcap\mathcal{S}''^{(i)}$ is $sp_5$. In what follows, we first present the detailed complexity formulas containing the probability parameters $p_5, p_6$ and then instantiate with $p_5 = p_6 = 1$. This corresponds to that each client uses the union of all chosen clients' real index sets to interact with the cloud server. As demonstrated in Theorem~\ref{theorem:same:security}, this case can provide the same levels of privacy and security as secure federated learning. Further, given $0 \leq p_5, p_6 \leq 1$ while the complexities are non-decreasing with $p_5, p_6$, this case still upper bounds the complexities of our scheme.

First regards the communication complexities of the client and the cloud server. Each client's communication overhead can be broken up as: (1) Participating in the private set union protocol, where two vectors need to be securely aggregated, and the final union result needs to be downloaded with size $O(ns)$. The first vector is the perturbed Bloom filter with size $\beta$, and the second vector is the perturbed indicator vector with a preset constant size. According to Equation~(\ref{eq:bloom:filter:optimal:length}), the optimal length of Bloom filter $\beta$ is proportional to the cardinality of the set to be filtered $\phi$, here the cardinality of the union of $n$ chosen clients' real index sets, which implies that $\beta \propto \phi = O(ns)$. Therefore, the communication complexity of each client in private set union is $O(n + ns + ns) = O(ns)$, where $O(n + ns)$ corresponds to the cost of securely aggregating the perturbed Bloom filters and is obtained by letting the size of vector $l$ in Table~\ref{tab:secure:aggregation:overhead} take $O(ns)$; (2) Downloading the training hyperparameters as well as the submodel with the perturbed index set $\mathcal{S}''^{(i)}$, namely a $(sp_5 + (n - 1)s p_6) \times d$ matrix; (3) Uploading the weighted submodel update and corresponding count vector with respect the perturbed index set through secure aggregation. In particular, the total size of vector to be securely aggregated with at most $n - 1$ other clients is $(sp_5 + (n - 1)s p_6)(d + 1)$. According to Table~\ref{tab:secure:aggregation:overhead}, the communication complexity of this part is $O(n + (sp_5 + (n - 1)s p_6)(d + 1))$. In summary, the total communication complexity of each client is $O(ns) + O((sp_5 + (n - 1)s p_6)d) + O(n + (sp_5 + (n - 1)s p_6)(d + 1)) = O(ns + (sp_5 + (n - 1)s p_6) (2d + 1))$. If the probability parameters $p_5, p_6$ both take the value $1$, the communication complexity of each client is $O(nsd)$. Correspondingly, the cloud server's communication cost can be broken up into: (1) Working as a mediation in the private set union protocol, and delivering the union result to clients with $O(n^2s)$ communication cost. Hence, the communication complexity of the cloud server in private set union is $O(n^2 + n^2s + n^2s) = O(n^2s)$, where $O(n^2 + n^2s)$ is related to secure aggregation; (2) Returning training hyperparameters and requested submodels to clients with $O(n + n(sp_5 + (n - 1)s p_6)d)$ communication cost; (3) Working as a mediation in the secure aggregations of weighted submodel updates and count vectors. Its communication complexity is $O(n^2 + n(sp_5 + (n - 1)s p_6)(d + 1))$. Overall, the cloud server's communication complexity is $O(n^2s) + O(n + n(sp_5 + (n - 1)s p_6)d) + O(n^2 + n(sp_5 + (n - 1)s p_6)(d + 1)) = O(n^2s + n(sp_5 + (n - 1)s p_6)(2d + 1))$. We still examine that when $p_5 = p_6 = 1$, the communication complexity of the cloud server is $O(n^2sd)$.

Second regards time complexities. Despite of the local training phase, the computation cost of each client is dominated by: (1) Perturbing the real index set, which costs $O(ns)$ time. Here, one set lookup operation can be implemented with $O(1)$ complexity; (2) Participating in the secure aggregation based stages, including private set union and secure weighted submodel updates and count vectors aggregations, which consume $O(n^2 + n^2s) = O(n^2s)$ and $O(n^2 + n(sp_5 + (n - 1)s p_6)(d + 1))$ time, respectively. Thus, the overall time complexity of each client is $O(ns) + O(n^2s) + O(n^2 + n(sp_5 + (n - 1)s p_6)(d + 1)) = O(n^2s + n(sp_5 + (n - 1)s p_6)(d + 1))$. When $p_5 = p_6 = 1$, it turns to $O(n^2sd)$. Correspondingly, the computation overhead of the cloud server is dominated by mediating private set union and secure aggregations of weighted submodel updates and count vectors, which consume $O(n^2 ns) = O(n^3s)$ and $O(n^2 (sp_5 + (n - 1)s p_6)(d + 1))$ time, respectively, and $O(n^3s + n^2 (sp_5 + (n - 1)s p_6)(d + 1))$ time in total. When $p_5 = p_6 = 1$, the cloud server's time complexity is $O(n^3sd)$.

Third regards storage complexities. Each client's storage overhead comes from: (1) Storing the materials in private set union and in secure aggregations of weighted submodel updates and count vectors, which occupy $O(n + ns) = O(ns)$ and $O(n + (sp_5 + (n - 1)s p_6)(d + 1))$ space, respectively; (2) Storing her permanent answers, which occupies $O(ns)$ space. Thus, the storage overhead of each client is $O(ns) + O(n + (sp_5 + (n - 1)s p_6)(d + 1)) + O(ns) = O(ns + (sp_5 + (n - 1)s p_6)(d + 1))$ in total. When $p_5 = p_6 = 1$, the client does not need to store the permanent answers any more, and her storage overhead is $O(nsd)$. Correspondingly, the cloud server's storage overhead can be broken up into: (1) Storing the materials in private set union and in secure weighted submodel updates and count vectors aggregations, which occupy $O(n^2 + ns)$ and $O(n^2 + (sp_5 + (n - 1)s p_6)(d + 1))$ space, respectively; (2) Storing $n$ chosen clients' perturbed index sets, which occupies $O(n(sp_5 + (n - 1)s p_6))$ space. In summary, the overall storage overhead of the cloud server is $O(n^2 + ns) + O(n^2 + (sp_5 + (n - 1)s p_6)(d + 1)) + O(n(sp_5 + (n - 1)s p_6)) = O(n^2 + ns + (sp_5 + (n - 1)s p_6)(n + d + 1))$. When $p_5 = p_6 = 1$, the cloud server does not need to store clients' perturbed index sets, and her storage overhead is $O(n^2 + nsd)$.

\subsubsection{Comparison with Secure Federated Learning}

We first analyze the complexities of secure federated learning. Each client's communication overhead mainly comes from: (1) Downloading the full model, namely a $m \times d$ matrix; (2) Uploading the update of the full model through the secure aggregation protocol with $O(n + md)$ complexity, which is obtained by letting the size of vector $l$ in Table~\ref{tab:secure:aggregation:overhead} take $md$. Thus, the overall communication complexity of each client is $O(n + md)$. In correspondence, the cloud server's communication cost comes from: (1) Sending the full model to all $n$ clients, with complexity $O(nmd)$; (2) Working as a mediation in the secure aggregation of the full model updates, with complexity $O(n^2 + nmd)$. Thus, the cloud server's communication complexity is $O(n^2 + nmd)$. Regardless of the local training phase, the computation and storage overheads of each client and the cloud server are dominated by secure aggregation. In particular, the time complexities of each client and the cloud server are $O(n^2 + nmd)$ and $O(n^2 md)$, respectively. Additionally, the storage complexities of each client and the cloud server are $O(n + md)$ and $O(n^2 + md)$, respectively.

We next compare our secure scheme with secure federated learning at the same security and privacy levels, \ie, all the probability parameters in our scheme are set to 1. Given that the cardinality of the union of $n$ chosen clients' real index sets is much smaller than the cardinality of the full index set (\ie, $\vert\bigcup_{i \in \mathcal{C}} \mathcal{S}^{(i)}\vert \ll \vert\mathcal{S}\vert \Rightarrow ns \ll m$), we can draw from Table~\ref{tab:fsl:fl:overhead} that the complexities of each client and the cloud server in our secure federated submodel learning are both much lower than those in secure federated learning. We can further derive from Table~\ref{tab:fsl:fl:overhead} that our secure scheme is quite scalable because its complexities are independent of the size of the full model, which is controlled by the number of total rows $m$.


We finally compare the overheads, typically the computation and memory overheads, of each client in the local training phase under two secure frameworks. We mainly focus on the size of local model/submodel. For secure federated learning, each client trains the full model $\mathbf{W}$ with size $md$. In contrast, for our secure federated submodel learning, each client trains the succinct submodel $\mathbf{W}_{\mathcal{S}^{(i)} \bigcap \mathcal{S}''^{(i)}}$ with size $sp_5d$, which is $sd$ when $p_5 = 1$. Given $md \gg nsd > sd \geq sp_5d$, we can find that our secure scheme is still far more efficient than secure federated learning in the local training phase.



\begin{table}[!t]
\caption{Statistics of Taobao dataset.} \label{tab:taobao:dataset}
\centering
\resizebox{\columnwidth}{!}{
\begin{tabular}[t]{l|cccc}
\toprule
Type & \#User(s) & \#Goods  &\#Categories &\#Samples \\
\midrule\midrule
Test (Full) & 24,790 & 138,829 & 4,758 & 1,010,284\\
Train (Full) & 49,023 & 143,534 & 4,815 & 15,854,357\\
Train (Per Client) & 1 & 301 & 117 & 323\\
\bottomrule
\end{tabular}
}
\end{table}

\section{Evaluation}

In this section, we introduce our evaluation results from model accuracy and convergency, practical computation, communication, and storage overheads.

{\bf Dataset:} We use an industrial dataset built from 30-day impression and click logs of Taobao users from June 15, 2019 to July 15, 2019. For a certain Taobao user, we leverage her click behaviors in previous 14 days as historical data to predict her click and non-click behaviors in the following 1 day. We leave out the samples within the last 1 day as the target items of the test set while putting the others into the training set. Specifically, the test set is located at the cloud server to judge the accuracy and convergency of the global model. In contrast, for the full training set, we further cluster each Taobao user's data as a training set located at a client. Some statistical information about the numbers of Taobao users, goods IDs, category IDs, and samples, in the test set, the full training set, and one client's training set is shown in Table~\ref{tab:taobao:dataset}.

{\bf Model, Hyperparameters, and Metrics:} We take the Deep Interest Network (DIN)~\cite{proc:alibaba:kdd18:din} as the e-commerce recommendation model for federated submodel learning, where the number of columns in the embedding matrix is set to 18. Except the embedding layer for user IDs, goods IDs, and category IDs, the parameters of the other network layers in DIN, including the attention layer and the fully connected layer, are of size 64,327. Hence, the global model at the cloud server is of size 3,617,023, whereas a desired submodel at the client is of size 71,869 in average, which is only $1.99\%$ of the global model's size and roughly requires $0.27$MB space using 32-bit representation. For each client's local training, we choose mini-batch SGD as the optimization algorithm, set the batch size to 2, and set the local epoch number to 1. In addition, we initially set the learning rate to 1 and further apply exponential decay with the decay rate of 0.999 per communication round. For the cloud server's global testing, we adopt a golden metric in the task of Click-Through Rate (CTR) prediction, called Area Under the Curve (AUC), and set the batch size to 1,024.

{\bf Prototypes and Configurations:} We implemented prototypes of our secure federated submodel learning and secure federated learning in Python 2.7.16. Due to the synchronization requirement of secure aggregation, we adopted a synchronous architecture on top as suggested by Google's deployment practice~\cite{proc:sysml19:fed}, where the cloud server delivers instructions to chosen clients, waits for them to finish their tasks, and then moves on to the next step. In particular, we implemented a communication module between the cloud server and each client with standard socket programming. In addition, we used TensorFlow 1.12.0 to implement DIN. Moreover, we mainly used PyCryptodome 3.7.3 to implement the secure aggregation protocol: Diffie-Hellman key exchange was implemented over RFC 3526 Group 14, which is a 2048-bit modular exponential (MODP) group~\cite{link:rfc3526}; the secret sharing adopted the standard Shamir's version; the authenticated encryption used Advanced Encryption Standard (AES) in the Cipher Block Chaining (CBC) mode, where the secret key is set to be 128-bit long; Pseudo-Random Number Generator (PRNG) was implemented using Hash-based Message Authentication Code (HMAC) suggested by NIST SP 800-90A~\cite{link:drbg}, and the security strength takes 128 bits.


Our running environment is a Linux workstation with 64-bit Ubuntu 18.04.2 OS. The processor is Intel(R) Core(TM) i9-9900K with 8 cores, the base frequency is 3.60GHz, the memory size is 64GB, and the cache size is 16MB. The workstation is also equipped with 2 NVIDIA's GeForce RTX 2080 Ti graphics cards. In our evaluation, to manifest the resource differentiation between clients and the cloud server, from hardware, we ran all the clients only on CPU, but allowed the cloud server to accelerate her operations using GPU. Additionally, from parallelism and concurrency, we optimized some of the cloud server's hotspot functions with the multiprocessing library in Python.



{\bf Implementation Overview:} We revisit some key points and also present more implementation details when our secure federated submodel learning is instantiated with Taobao's e-commerce recommendation. The cloud server chooses $n$ clients, namely $n$ Taobao users, in every communication round, where $n$ can range from 20 to 100 with a step of 20. Each chosen client extracts goods IDs from her data as her real index set. Then, she joins in private set union to obtain the union of $n$ chosen clients' goods IDs. In our Taobao dataset, the maximum size of the union of 100 clients' goods IDs reaches 32,904, which corresponds to the cardinality of set to be filtered $\phi$. From Equation~(\ref{eq:bloom:filter:optimal:length}), we can derive that the optimal length of Bloom filter $\beta$ should be 630,774 at a desired false positive rate of $0.01\%$. This is larger than the total number of goods IDs 143,534 involved in this Taobao dataset. Thus, we set the length of Bloom filter to be 143,534, take an identity map as the only hash function, and omit the partition steps. The false positive rate now is 0. In addition, we set the modulus $R$ to $2^{32}$ in private set union, and any false negative due to modular addition in recovering union never happened throughout our evaluation. It is worth noting that our evaluation results of private set union presented here can sufficiently embody the practical case, where the full size of the goods IDs in Taobao scales to billions. The reason is that as shown in Table~\ref{tab:fsl:fl:overhead}, the overheads of our private set union scheme only hinge on the number of chosen clients $n$ and the optimal length of Bloom filter $\beta$, where the latter is independent of the domain of the represented set's elements and is at the same level of 143,534 evaluated here.



\begin{table}[!t]
\caption{Choices of Probability Parameters (CPPs) and resulting privacy levels.}\label{tab:probabilities:choices:rough}
\centering
\resizebox{\columnwidth}{!}{
\begin{tabular}[t]{l|cc|cc|cc|cc}
\toprule
     &$p_1, p_3$ & $p_2, p_4$ & $p_5$ & $p_6$ & $\epsilon_1$ & $\epsilon_\infty$ & $p_{7}$ & $p_{8}$\\
\midrule\midrule
CPP1 & 1 & 0 & 1 & 0 & $\infty$ & $\infty$ & $86.7\%$ & 0\\
CPP2 & $15/16$ & $1/16$ & $88.3\%$ & $11.7\%$ & $2.02$ & $2.71$ & 0 & $10.3\%$\\
CPP3 & $7/8$ & $1/8$ & $78.1\%$ & $21.9\%$  & $1.27$ & $1.95$ & 0 & $19.5\%$\\
CPP4 & $3/4$ & $1/4$ & $62.5\%$ & $37.5\%$ & $0.51$ & $1.10$ & 0 & $34.2\%$\\
CPP5 & 1 & 1 & 1 & 1 & 0 & 0 & 0 & 0\\
\bottomrule
\multicolumn{9}{l}{*Smaller $\epsilon_1$, $\epsilon_\infty$, $p_{7},$ and $p_{8}$ indicate better privacy.}
\end{tabular}
}
\end{table}


After obtaining the union of goods IDs, each client runs Algorithm~\ref{alg:index:set:ldp} to generate a set of perturbed goods IDs. For clarity in presenting results, we let each client use the same Choice of Probability Parameters (CPP). Table~\ref{tab:probabilities:choices:rough} lists 5 CPPs in our evaluation and their resulting privacy levels of index set perturbation and secure submodel updates aggregation, where the number of chosen clients is 100 for the secure submodel updates aggregation. We provide some insights from Table~\ref{tab:probabilities:choices:rough} as follows: (1) CPP1 corresponds to the worst local differential privacy guarantee where each client reveals her real goods IDs, while CPP5 corresponds to the best privacy guarantee as strong as secure federated learning where each client uses the union of chosen clients' goods IDs as her perturbed index set. Additionally, as the serial number of CPP becomes larger, the local differential privacy guarantee is stronger; (2) the resulting $p_{7} = 86.7\%$ at CPP1 indicates that if each chosen client submits the submodel update using her real goods IDs, then Event 1 (\ie, from the securely aggregated submodel update, the cloud server ascertains that some goods ID belongs to a concrete client and also learns the detailed update with respect to this goods ID.) happens with probability $86.7\%$. This further implies that $86.7\%$ of the goods IDs in the union of 100 chosen clients' real goods IDs involve single client. Therefore, we can draw that user data in our e-commerce context is highly heterogeneous, and the truly required submodels of different clients are highly differentiated; (3) with CPP from CPP2 to CPP5, the cloud server cannot ensure that any goods ID belongs to some client from the aggregate submodel update and also cannot learn any client's detailed update, namely Event 1 happens with probability $p_{7} = 0$; and (4) with CPP from CPP2 to CPP4, the proportion of each client's real goods IDs falling into her perturbed goods IDs controlled by $p_5$ decreases, while the proportion of redundant goods IDs controlled by $p_6$ increases. Hence, the proportion of aggregate zero updates grows, implying a higher probability $p_{8}$ of the cloud server in ascertaining that a certain goods ID does not belong to some client. Further, given an observation that $p_{8}$ is approaching $1 - p_5$, we can still draw that the real goods IDs of different Taobao users are highly differentiated.

Based on the perturbed goods IDs, each client can further generate the perturbed category IDs by leveraging the global mapping between goods IDs and category IDs, which is publicly shared among all clients and the cloud server. Upon receiving the perturbed goods IDs from a client, the cloud server can still generate her perturbed category IDs, thus returning a submodel to the client. Here, the submodel returned to the client comprises the embedding vectors for her perturbed goods and category IDs as well as the other network parameters of DIN. In addition, the client should not generate perturbed goods and category IDs independently by running Algorithm~\ref{alg:index:set:ldp} twice. The reason is that there exist strong correlations, particularly a surjective but not injective mapping, between goods IDs and category IDs, whereas independent randomized responses to dependent questions can lead to inconsistency and may leak information about true answers to the adversary. Therefore, we should apply randomized response to major questions and then generate answers to secondary/correlated questions via their dependence relationships, which can keep consistency. For example, suppose you are asked two types of questions: one type is about whether you have a certain fruit by enumerating all possible fruits, such as ``Do you have an apple?" or ``Do you have a banana?"; the other type is about whether you have the category fruit. To preserve plausible deniability and consistency of your answers, you should simply apply randomized response to the first type of questions, and then directly respond to the second question based on the noisy answers to the first type of questions. In other words, if you respond any ``Yes" to the first type of questions, then you should respond ``Yes" to the second question; otherwise, you should respond ``No".

By performing set intersection between real and perturbed goods IDs, each client obtains her succinct set of goods IDs and also gets her succinct set of category IDs based on the global goods-category mapping. Then, the client extracts her succinct submodel from the downloaded submodel and also extracts her succinct training set from the original training set by following two rules: (1) For the goods ID to be predicted in a sample, if it does not belong to the succinct set of goods IDs, this sample will be filtered out; and (2) for the sequence of historical goods IDs in a sample, we only keep those goods IDs in the succinct set of goods IDs as well as their corresponding category IDs. Of course, if none of goods IDs are left in the historical sequence, this sample will be filtered out. Afterwards, each client trains her succinct model, obtains the update of the submodel, and prepares the weighted submodel update and the count vector to be uploaded.

To facilitate the cloud server in obliviously adding up the weighted submodel updates based on secure aggregation, the float-type parameters need to be converted to integers so that modular addition is supported. One common practice is that each client first scales up the parameters through multiplying by a large constant (\eg, a power of 2), and only keeps the integral parts. Later, the cloud server scales down the aggregate result through dividing by the same constant. Different from the common scaling technique, in this work, we perform float-to-integer conversion in a new manner, particularly by means of a celebrated model compression algorithm, called stochastic $\gamma$-level quantization~\cite{proc:icml17:kquantization}, which maps each parameter into $\{0, 1, \ldots, 2^{\gamma} - 1\}$. The detailed procedure is shown as follows: Each client first compresses her submodel update and then weights the parameters in the compressed submodel update by multiplying with corresponding count numbers. After the secure aggregation, the cloud server gets the sum of the weighted compressed submodel updates, and then divides by corresponding total count numbers. Finally, the cloud server applies the decompression algorithm and updates the global model by adding the aggregate submodel update. We note that the postprocessing with a weighted average/mean in the compressed space does not introduce any bias/error to the final decompressed result. Detailed proof is deferred to Appendix~\ref{app:fedavg:quantization}. Furthermore, in our following evaluation, we set $\gamma$ to $2^{15}$ and set the modulus in the secure aggregation of weighted compressed submodel updates to $2^{32}$.

\begin{figure}[!t]
\centering
\includegraphics[width = 0.95\columnwidth]{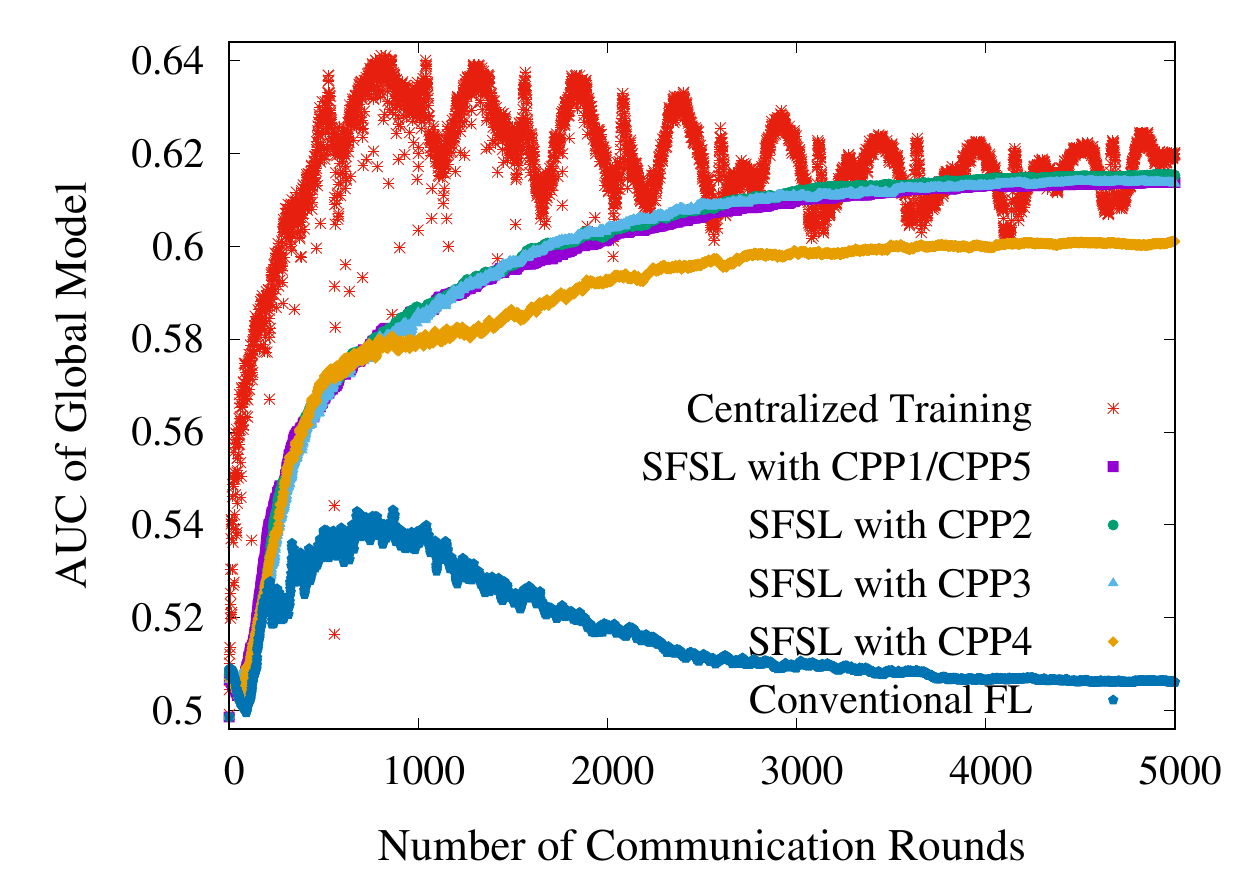}
\caption{Accuracies and convergencies of global model in centralized training, Secure Federated Submodel Learning (SFSL) with different Choices of Probability Parameters (CPPs), and conventional Federated Learning (FL).}\label{fig:global:model:aucs}
\end{figure}

\begin{figure*}[t]
\centering
\subfigure[Per Client]{\label{fig:comm:overhead:total:client}
\includegraphics[width=0.66\columnwidth]{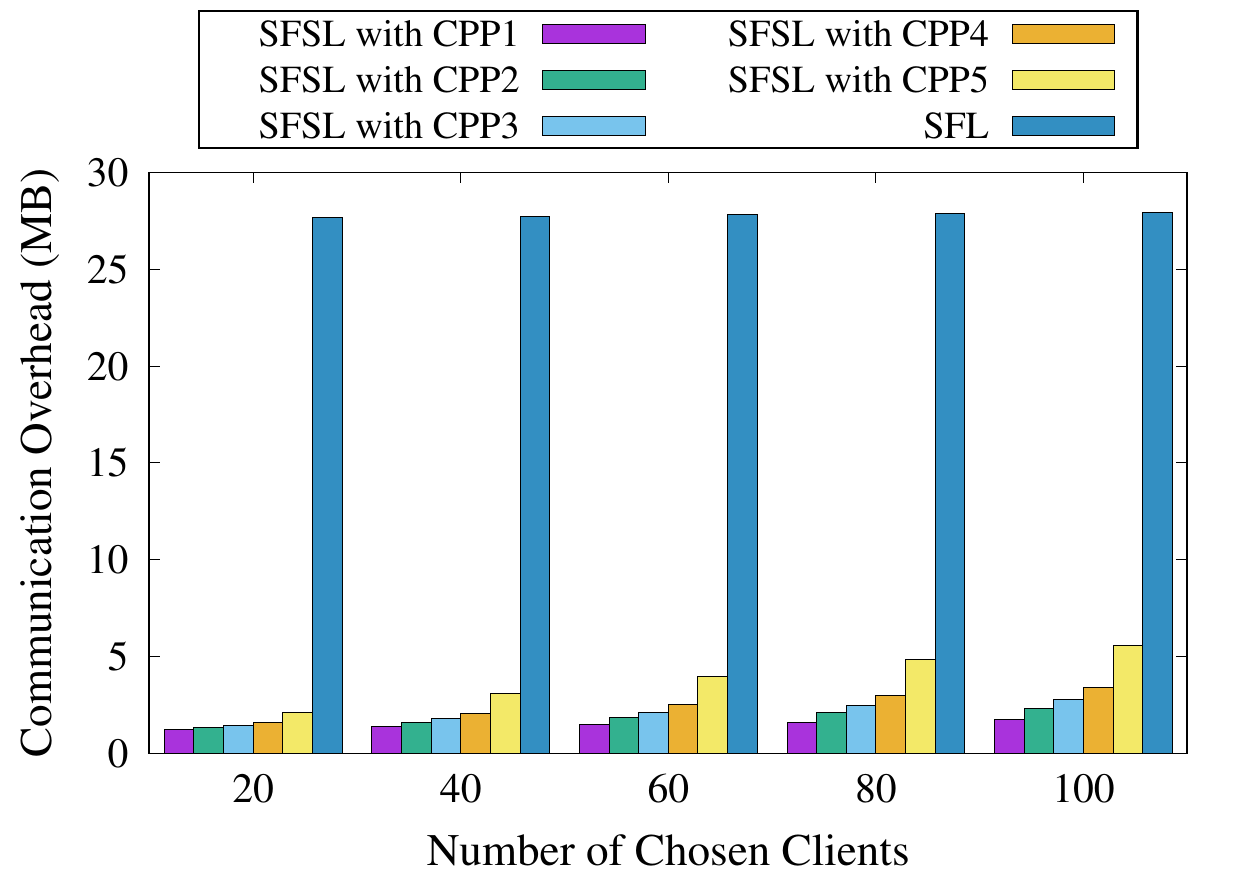}}
\subfigure[Per Client]{\label{fig:comp:overhead:totoal:client}
\includegraphics[width=0.66\columnwidth]{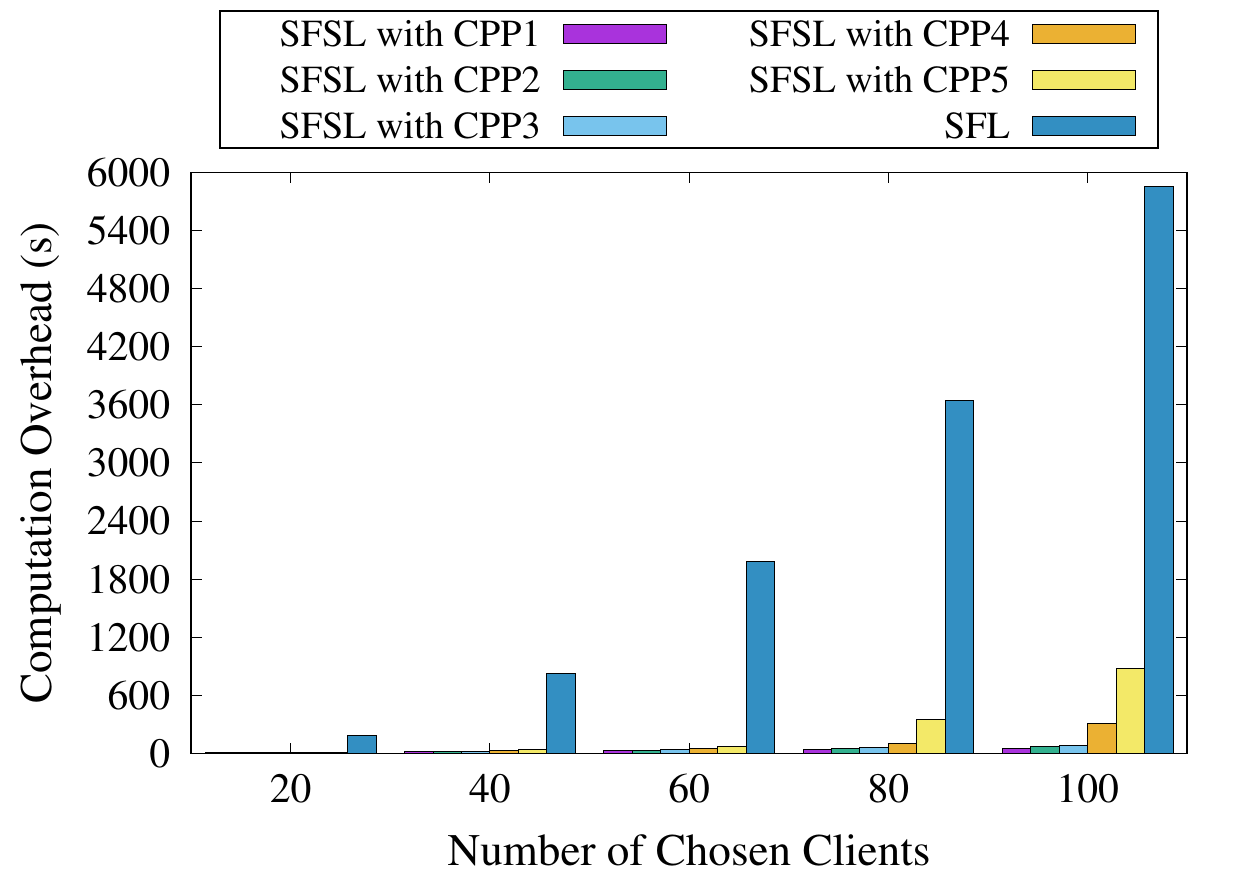}}
\subfigure[Cloud Server]{\label{fig:comp:overhead:totoal:server}
\includegraphics[width=0.66\columnwidth]{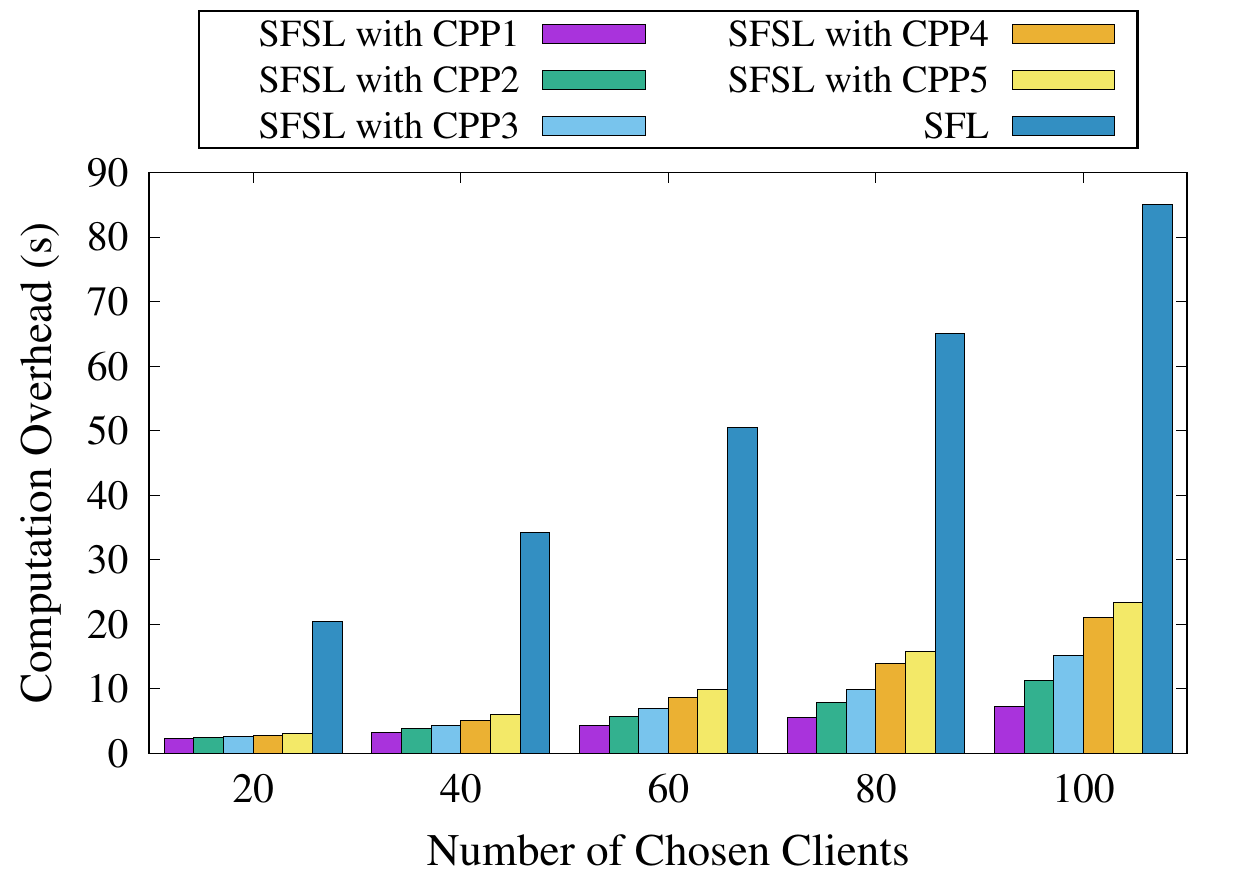}}
\caption{Total communication and computation overheads of the client and the cloud server per communication round in Secure Federated Submodel Learning (SFSL) with different Choices of Probability Parameters (CPPs) and Secure Federated Learning (SFL).}\label{fig:comm:comp:overhead:totoal}
\end{figure*}

\subsection{Model Accuracy and Convergency}\label{sec:evaluation:accuracy}

We bring in centralized training and conventional federated learning as two baselines and plot their AUCs as well as the AUCs of our secure federated submodel learning with different CPPs in Fig.~\ref{fig:global:model:aucs}. Here, the number of chosen clients in each communication round $n$ is set to 100, and the total number of communication rounds is set to 5,000. In addition, centralized training refers to the traditional case that the cloud server first collects data from all clients, then trains the DIN model, and tests the model once training over the samples with a similar size to the total size of $n$ chosen clients' datasets.


From Fig.~\ref{fig:global:model:aucs}, we can see that compared with centralized training, which reaches the highest AUC of 0.641 in 803 communication rounds, our secure federated submodel learning with CPP2 achieves the highest AUC of 0.615 in 4,908 communication rounds, slightly decreasing by 0.026. In contrast, conventional federated learning performs worst among all schemes, only achieving the highest AUC of 0.543 in 867 communication rounds. Even worse, it does not converge to a good model at the end of many communication rounds\footnote{We have tried several different pairs of an initial learning rate and a decay rate and all observed divergences in conventional federated learning. For example, when the initial learning rate is 4, and the exponential decay rate is 0.996, conventional federated learning achieves the best AUC of 0.554 in 230 rounds but diverges to 0.503 at the end of 5,000 rounds.}. The major reason for conventional federated learning not working well in the e-commerce recommendation context is that it coarsely computes the weighted average of the clients' updates of the full model proportional to their training set sizes, no matter whether one client's whole training set actually involves some network parameters (the full model excluding her submodel, \eg, some embedding vectors here), thus inaccurately counting in the weights (\ie, the training set sizes) of those clients who contribute zero/no updates for these network parameters. In addition, with a higher heterogeneity of user data and thus a higher differentiation of submodels, the roughness and inaccuracy of conventional federated learning will be exposed more completely, which clarifies why it can work in the natural language context with a 10,000 word vocabulary considered by Google, but does not work well in our e-commerce context with billion-scale goods IDs. More thorough demonstrations are deferred to Appendix~\ref{app:divergence}.

From Fig.~\ref{fig:global:model:aucs}, we can also observe some differences among the AUCs of our secure federated submodel learning with different CPPs. In particular, secure federated submodel learning with CPP4 is the worst among all CPPs, and it achieves the highest AUC of 0.601, decreasing by $0.014$ compared with CPP2. We clarify the reason through the resulting probabilities $p_5$'s of different CPPs listed in Table~\ref{tab:probabilities:choices:rough}, where $p_5$ denotes the probability that an index in a client's real index set finally is put into her perturbed index set, and it dominates the size of the client's succinct training set. Further considering the process of generating the succinct training set in our evaluation, in addition to the size of each client's factual local training samples, $p_5$ also controls the length of historical goods and category IDs in every sample. Thus, a smaller $p_5$ tends to imply worse model performance in general. This accounts for different model performances under CPP2 and CPP4, and also explains the observation that CPP1 and CPP5, sharing the same $p_5 = 1$, have identical model performance.

\subsection{Communication Overhead}

We show the total communication overhead of secure federated submodel learning and first introduce secure federated learning as a baseline. Fig.~\ref{fig:comm:overhead:total:client} plots the communication overhead per client per communication round. Here, we do not plot the communication overhead of the cloud server, since it is equal to the communication overhead per client multiplying by the number of chosen clients. In more detail, the incoming data of the cloud server is exactly the total outgoing data of all chosen clients, and vice verse. Additionally, we also do not plot for different dropout ratios because this factor has little impact on the bandwidth cost.

One key observation from Fig.~\ref{fig:comm:overhead:total:client} is that our secure federated submodel learning can significantly reduce the total communication overhead, compared with secure federated learning. In particular, when the number of chosen clients is 100, the communication overheads per client per communication round are 1.76MB, 2.33MB, 2.78MB, 3.40MB, and 5.57MB in secure federated submodel learning with CPP1, CPP2, CPP3, CPP4, and CPP5, respectively, reducing $93.72\%$, $91.65\%$, $90.06\%$, $87.81\%$, and $80.05\%$ than secure federated learning, which incurs 27.94MB per client per round. Considering secure federated submodel learning with CPP5 share the same levels of security and privacy with secure federated learning (\ie, Theorem~\ref{theorem:same:security}), we can draw that our secure scheme can reduce communication overhead even with not scarifying any security or privacy. These results coincide with our complexity analysis in Section~\ref{sec:performance:analysis} and Table~\ref{tab:fsl:fl:overhead}.

The second key observation from Fig.~\ref{fig:comm:overhead:total:client} is that for secure federated submodel learning with a certain CPP, the communication overhead per client increases with the number of chosen clients $n$. In addition, for a certain number of chosen clients, the communication overhead per client increases with the serial number of CPP. We clarify the reasons by adopting the detailed communication complexity formula of each client from Section~\ref{sec:performance:analysis}: $O(ns + (sp_5 + (n - 1)s p_6) (2d + 1))$. On one hand, the communication complexity grows linearly with $n$. On other other hand, it is increasing with $p_5$ and $p_6$, and thus CPP5 is most communication expensive. Additionally, given $p_5 + p_6 = 1$ for CPPs from CPP1 to CPP4 in Table~\ref{tab:probabilities:choices:rough}, we can simplify the formula to $O(ns + (s + (n - 2)s p_6)(2d + 1))$, which increases with $p_6$ for $n > 2$. From Table~\ref{tab:probabilities:choices:rough}, we can see that $p_6$ increases with the serial number of CPP, implying a higher communication overhead as depicted in Fig.~\ref{fig:comm:overhead:total:client}. Intuitively, $p_6$ denotes the probability that an index not in a client's real index set finally falls into the perturbed index set and controls the size of the redundant/zero parameters to be downloaded and securely uploaded. Thus, when the sum of two probabilities is fixed, particularly $p_5 + p_6 = 1$, the increase of bandwidth cost due to introducing redundant parameters exceeds the decrease due to discarding original parameters. In other words, the introduced redundant parameters controlled by $p_6$ dominates the holistic trend of communication overhead.

We next introduce the pure versions of federated submodel learning and conventional federated learning shown in Fig.~\ref{fig:system:model} as another type of baselines, and investigate the expansion factors due to the security and privacy guarantees. In particular, the pure federated submodel learning (\resp, federated learning) means that each client directly downloads her required submodel (\resp, the full model) from the cloud server, and then uploads the submodel update and the count vector (\resp, the full model update and the training set size) to the cloud server. The communication overheads per client per round are 0.41MB and 20.70MB in the federated submodel learning and federated learning, respectively, and are irrelevant with the number of chosen clients. Compared with the pure version, when the number of chosen clients is 100, the communication overhead of secure federated submodel learning with CPP2 (\resp, secure federated learning) expands $5.65\times$ (\resp, $1.35\times$). Three major reasons account for a larger expansion factor in secure federated submodel learning: (1) First is that the bandwidth cost of pure federated submodel learning, as the denominator, is much lower than, particularly $1.99\%$ of, the pure federated learning's bandwidth cost; (2) second is that the size of model parameters in secure federated learning is much larger than that in secure federated submodel learning, which can amortize the communication cost spent in transferring security and privacy related parameters; and (3) third is that secure federated submodel learning requires an extra process of private set union to facilitate later index set perturbation.


\begin{table}[!t]
\caption{Communication and computation overheads of the client and the cloud server per round in private set union.} \label{tab:psu:comm:overhead}
\centering
\resizebox{\columnwidth}{!}{
\begin{tabular}[t]{l|ccccc}
\toprule
\#Chosen Clients $n$ & 20 & 40 & 60 & 80 & 100\\
\midrule\midrule
Client's Comm. Overhead (MB) & 0.63 & 0.70  & 0.77  & 0.84  & 0.91\\
Client's Comp. Overhead (s)  & 5.04 & 10.09 & 15.53 & 26.49 & 33.69\\
Server's Comp. Overhead (s)  & 1.12 & 1.68  & 2.40  & 3.13  & 4.19\\
\bottomrule
\end{tabular}
}
\end{table}

We finally present the communication overhead per client per round of our private set union protocol. Table~\ref{tab:psu:comm:overhead} lists the detailed bandwidth cost. We can see that the communication overhead per client in private set union increases linearly with the number of chosen clients, roughly with an increase of 0.07MB per 20 clients. In addition, we can also see that our private set union is communication efficient, and it only incurs 0.91MB when the number of chosen clients reaches 100. These evaluation results conform to our complexity analysis in Section~\ref{sec:performance:analysis} and Table~\ref{tab:fsl:fl:overhead}.

\subsection{Computation Overhead}

We now report the practical computation overhead, mainly by investigating the effects of the number of chosen clients, the choices of probability parameters, as well as the ratios of dropout. To be consistent with our time complexity analysis, the computation overhead here only includes the run time of the client or the cloud server in executing the protocol but ignores synchronization delay and the time overhead of testing global model. Of course, given mobile clients are highly parallel in practice, the total run time per communication round can be estimated by adding up the computation overheads of the client and the cloud server shown here. In addition, testing the global model per round costs the cloud server 32.12s.



We first show the computation overheads of the client and the cloud server in our secure federated submodel learning with different CPPs in Fig~\ref{fig:comp:overhead:totoal:client} and Fig~\ref{fig:comp:overhead:totoal:server}, respectively. We still introduce secure federated learning as a baseline. First, we compare two figures and find that the computation overhead per client is higher than that of the cloud server. For example, in secure federated submodel learning with CPP2, when the number of chosen clients is 100, the computation overheads of the client and the cloud server are 71.80s and 11.31s, respectively. In addition to the superiorities of hardware and multiprocessing, under the synchronous architecture, the cloud server actually takes up the entire computing resources of the physical workstation alone, which are instead shared by all $n$ chosen clients simultaneously. These factors jointly account for the differences between Fig~\ref{fig:comp:overhead:totoal:client} and Fig~\ref{fig:comp:overhead:totoal:server}. Second, we focus on a certain side, either the client or the cloud server, and we can observe that her computation overhead grows with the number of chosen clients or the serial number of CPP. Here, we explain the reason by means of the detailed time complexities in Section~\ref{sec:performance:analysis}. For example, the time complexity of the client is $O(n^2s + n(sp_5 + (n - 1)s p_6)(d + 1))$, which is increasing with the number of chosen clients $n$ as well as $p_5$ and $p_6$, where the latter explains why CPP5 is the most time-consuming one among all CPPs. Regarding CPP1 to CPP4 where $p_5 + p_6 = 1$, we can simplify the complexity formula to $O(n^2s + n(s + (n - 2)s p_6)(d + 1))$, which is increasing with $p_6$ for $n > 2$. The intuition behind the above analysis is analogous to that behind communication overhead, \ie, the size of redundant/zero parameters dominates the holistic computation overhead. Third, we compare our secure federated submodel learning with secure federated learning, and can find that our secure scheme significantly outperforms the baseline on both the sides of the client and the cloud server. In particular, when the number of chosen clients is 100, at the same security and privacy levels, secure federated submodel learning with CPP5 reduces $85.02\%$ and $72.51\%$ of computation overheads than secure federated learning on the sides of the client and the cloud server, respectively. When security and privacy become weaker, the advantages of our scheme are more evident under CPP1 to CPP4, \eg, CPP2 reduces $98.77\%$ and $86.70\%$ on the sides of the client and the cloud server, respectively.


\begin{figure}[!t]
\centering
\includegraphics[width = 0.95\columnwidth]{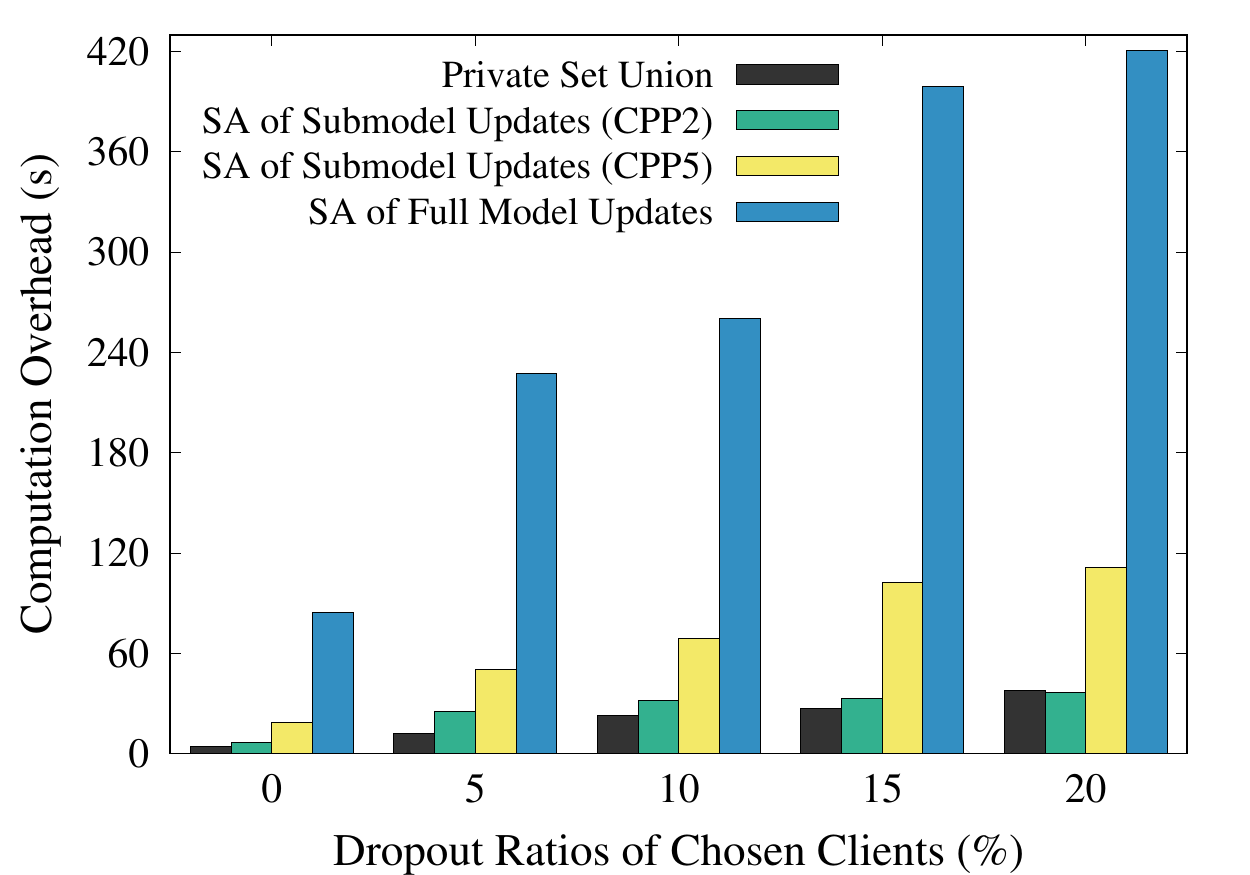}
\caption{Computation overheads of the cloud server with varying dropout ratios of chosen clients, in private set union, Secure Aggregations (SAs) of the weighted submodel updates and count vectors in SFSL with different CPPs, and Secure Aggregation (SA) of the full model updates in SFL.}\label{fig:comp:overhead:server:dropout}
\end{figure}

We next investigate the effect of the dropout ratio of chosen clients and depict the evaluation results in Fig.~\ref{fig:comp:overhead:server:dropout}. Here, we mainly focus on the secure aggregation based stages while ignoring the other stages which are irrelevant with dropout. In addition, we fix the number of chosen clients in each communication round at 100. Moreover, each client can randomly go offline after sending the encrypted shares of her private PRNG seed for the self mask and her private key for the mutual mask to the other clients because only this type of dropout behaviors introduce additional overhead. In particular, live clients still incorporate the public keys of dropped clients for mutual masks, and the cloud server must perform an expensive recovery computation to remove these mutual masks. The higher the dropout ratio is, the more costly the recovery computation will be. This trend is clear in Fig.~\ref{fig:comp:overhead:server:dropout}. Last, we only report the computation overhead of the cloud sever because dropped clients do not introduce additional operation cost to live clients, like for communication overhead. Of course, the case that more clients are dropped can mitigate the competition for system resources, especially when the number of chosen clients is large and the size of vector for secure aggregation is large. For example, when the number of chosen clients is 100 and the dropout ratio is $20\%$, for the secure aggregations of weighted submodel updates and count vectors in our secure scheme with CPP5, the computation overhead per client reduces by $16.75\%$ than the overhead in the case of no dropout.

We further observe Fig.~\ref{fig:comp:overhead:server:dropout} more carefully. First, we compare our secure scheme with secure federated learning in the stages of obliviously summing submodel or full model updates. We can find that our scheme significantly outperforms secure federated learning at any dropout ratio. Specifically, when the dropout ratio is $20\%$, our scheme with CPP2 and CPP5 reduce $91.33\%$ and $73.55\%$ of computation overhead, respectively. Second, we examine our private set union protocol and can see that it is quite efficient, even when the dropout ratio is high. In particular, when the dropout ratio reaches $20\%$, the computation overhead of the cloud server is 37.66s.

\subsection{Memory and Disk Loads}

We finally present the practical storage overheads of secure federated submodel learning and secure federated learning, including memory and disk loads. In particular, the materials for private set union and secure aggregations of submodel/full model updates and count vectors are stored in memory for immediate use, whereas permanent answers of each client are written into disk if necessary.

First is about memory overhead. The cloud server requires the video memory of 551MB, mainly for testing the global model at the end of each communication round, which is the same for all schemes. In addition, when the number of chosen clients is 100, and there is no dropout, the memory overheads per client are 209MB and 281MB in our secure federated submodel learning with CPP2 and CPP5, respectively, reducing $59.40\%$ and $45.43\%$ than secure federated learning. Correspondingly, the memory overheads of the cloud server are 1.58GB and 3.15GB, respectively, reducing $81.88\%$ and $63.77\%$ than secure federated learning. Furthermore, compared with the pure version, the memory overheads of our secure federated submodel learning with CPP2 and CPP5 only slightly expand on the client's side, and expand $1.19\times$ and $2.38\times$ on the cloud server's side, respectively. In contrast, the expansion factor is $4.60\times$ on the cloud server's side in secure federated learning.

Regarding the disk load, only the client in our scheme with CPP2, CPP3, and CPP4 needs to store her permanent answers in index set perturbation for multiple communication rounds, which roughly occupies the disk space of 280KB within the total 5,000 rounds.


\subsection{Summary and Discussion}

The above evaluation results adequately demonstrate the effectiveness and efficiency of our proposed secure federated submodel learning as well as its superiority over the baselines of conventional federated learning. Additionally, when the size of the full model, depending on the full size of the goods IDs in the e-commerce scenario, scales further to billions in practice, it has no effect on the performance and overhead of our scheme. This is because as analyzed in Section~\ref{sec:performance:analysis} and summarized in Table~\ref{tab:fsl:fl:overhead}, the complexities of our scheme are independent of the size of the full model. However, the conventional federated learning framework, hinging on the full model, will be too prohibitively inefficiently to be applicable.

\section{Conclusion}

In this paper, we have proposed a new framework, called secure federated submodel learning, for numerous clients to effectively and efficiently train large-scale deep learning models under the coordination of an untrusted cloud server while keeping their user data private. We further have applied our framework to the e-commerce recommendation scenario of Alibaba, implemented a prototype system, and extensively evaluated its performance over a Taobao dataset. Evaluation results have validated practical feasibility.

\begin{appendix}

\subsection{Fine-Tuning the Privacy Level of Secure Submodel Updates Aggregation}\label{app:tune:submodel:aggregation}
We introduce how to enable each client to fine-tune the privacy level of secure submodel updates aggregation by analyzing the impacts of the probability parameters $p_5$ and $p_6$ on the probabilities of Event 1 and Event 2, \ie, $p_{7} = p_5 (1 - p_5)^{n_{j, 1} - 1} (1 - p_6)^{n_{j, 0}}$ and $p_{8} = (1 - p_5)^{n_{j, 1}}(1 - (1 - p_6)^{n_{j,0}})$, respectively. Without loss of generality, we consider three different policies: (1) Fixing $p_5$ and adjusting $p_6$; (2) fixing $p_6$ and adjusting $p_5$; and (3) fixing $p_5 + p_6 = 1$. In particular, as shown in Table~\ref{tab:probabilities:choices:rough}, we mainly adopted the first and the third policies in our evaluation.

We first analyze $p_{7}$ as follows: (1) If $p_5$ is fixed, $p_{7}$ decreases as $p_6$ increases; (2) If $p_6$ is fixed, the monotonicity is nontrivial. We need to compute the derivative of $p_{7}$ with respect to $p_5$, \ie,
\begin{align*}
\frac{\mathrm{d} p_{7}}{\mathrm{d} p_5} = \left(1 - p_5 n_{j,1}\right) \left(1 - p_5\right)^{n_{j,1} - 2} \left(1 - p_6\right)^{n_{j,0}}.
\end{align*}
Thus, if $p_5 < 1/n_{j,1}$, $p_{7}$ increases as $p_5$ increases; otherwise, $p_{7}$ decreases as $p_5$ increases; (3) If $p_5 + p_6 = 1$, we first simplify $p_{7}$ into ${p_5}^{n_{j, 0} + 1} (1 - p_5)^{n_{j, 1} - 1} $. Then, we compute the derivative of $p_{7}$ with respect to $p_5$ as
\begin{align*}
\frac{\mathrm{d} p_{7}}{\mathrm{d} p_5} = \left(n_{j,0} + 1 - \left(n_{j,0} + n_{j,1}\right)p_5\right) {p_5}^{n_{j, 0}} \left(1 - p_5\right)^{n_{j,1} - 2}.
\end{align*}
Hence, if $p_5 < (n_{j, 0} + 1)/(n_{j, 0} + n_{j,1})$, $p_{7}$ increases as $p_5$ increases; otherwise, $p_{7}$ decreases as $p_5$ increases.

We next analyze $p_{8}$ as follows: (1) If $p_5$ is fixed, $p_{8}$ increases as $p_6$ increases; (2) If $p_6$ is fixed, $p_{8}$ decreases as $p_5$ increases; (3) If $p_5 + p_6 = 1$, we simplify $p_{8}$ into $(1 - p_5)^{n_{j, 1}}(1 - {p_5}^{n_{j,0}})$. Thus, $p_{8}$ decreases as $p_5$ increases.

Finally, we can verify the above deductions by checking the differences among CPPs and the changes in the resulting privacy levels of secure submodel updates aggregation in Table~\ref{tab:probabilities:choices:rough}, where the number of chosen clients is 100 (\ie, $n_{j, 0} + n_{j, 1} = 100$), and the number of chosen clients who have an arbitrary goods ID from the union is $1.17$ in average (\ie, $n_{j, 0} = 98.83, n_{j,1} = 1.17$).

\subsection{$\gamma$-Level Stochastic Quantization and Weighted Average}\label{app:fedavg:quantization}

We first briefly review the application of $\gamma$-level stochastic quantization in secure federated submodel learning. Considering that the compression, federated averaging (or mathematically, weighted average/mean), and decompression of submodel updates are element-wise, we thus just focus on the update of one parameter. Here, we consider that the parameter update and the count number from client $i \in \mathcal{C}$ are $\Delta w^{(i)} \in \mathbb{R}$ and $v^{(i)} \geq 0 \in \mathbb{Z}$, respectively, and the cloud server wants to compute $\sum_{i\in\mathcal{C}} v^{(i)} \Delta w^{(i)} /\sum_{i\in\mathcal{C}} v^{(i)}$. In addition, we assume that $\forall i \in \mathcal{C}, \Delta w_{\min} \leq \Delta w^{(i)} \leq \Delta w_{\max}$. Moreover, for $\gamma$-level quantization, the interval from $\Delta w_{\min}$ to $\Delta w_{\max}$ should be equally divided into $\gamma - 1$ segments, where the length of each segment is $\Delta w_{\mathrm{unit}} = (\Delta w_{\max} - \Delta w_{\min})/(\gamma - 1)$. In fact, to compress $\Delta w^{(i)}$, client $i$ needs to find the segment that contains $\Delta w^{(i)}$, and then takes either the starting index or the ending index of the segment with probability (w.p.) inversely proportional to the distance between $\Delta w^{(i)}$ and the starting or ending point. More specifically, $\Delta w^{(i)}$ is mapped into $z^{(i)} \in \{0, 1, \ldots, \gamma - 1\}$, where
\begin{equation*}
z^{(i)} = \left\{
\begin{aligned}
&\left\lfloor \frac{\Delta w^{(i)} - \Delta w_{\min}}{\Delta w_{\mathrm{unit}}} \right\rfloor\\
&\quad\quad \mathrm{w.p.}\ \left\lceil \frac{\Delta w^{(i)} - \Delta w_{\min}}{\Delta w_{\mathrm{unit}}} \right\rceil - \frac{\Delta w^{(i)} - \Delta w_{\min}}{\Delta w_{\mathrm{unit}}},\\
&\left\lceil  \frac{\Delta w^{(i)} - \Delta w_{\min}}{\Delta w_{\mathrm{unit}}} \right\rceil\ \mathrm{otherwise}.
\end{aligned}
\right.
\end{equation*}
On one hand, we can compute that the expectation of $z^{(i)}$ is $\frac{\Delta w^{(i)} - \Delta w_{\min}}{\Delta w_{\mathrm{unit}}}$, which is denoted as $z^{(i)*}$ and will be consistently used in our following derivations. On the other hand, we can decompress $z^{(i)}$ and get the recovered parameter as $z^{(i)} \Delta w_{\mathrm{unit}} + \Delta w_{\min}$, the expectation of which is $\Delta w^{(i)}$.

We next demonstrate that the weighted averaging operation in the compression space will not introduce any bias/error. On each client's side, she compresses her parameter update into:
\begin{align*}
\forall i \in \mathcal{C}, z^{(i)*} = \frac{\Delta w^{(i)} - \Delta w_{\min}}{\Delta w_{\mathrm{unit}}}.
\end{align*}
Then, she further weights $z^{(i)*}$ through multiplying by $v^{(i)}$ (\ie, $v^{(i)}z^{(i)*}$), and uploads the materials for secure aggregation. On the cloud server's side, she divides the aggregate result $\sum_{i \in \mathcal{C}} v^{(i)}z^{(i)*}$ by $\sum_{i\in\mathcal{C}} v^{(i)}$, gets $\sum_{i \in \mathcal{C}} v^{(i)}z^{(i)*} / \sum_{i\in\mathcal{C}} v^{(i)}$, and finally performs decompression as
\begin{align*}
  &\frac{\sum_{i \in \mathcal{C}} v^{(i)}z^{(i)*}}{\sum_{i\in\mathcal{C}} v^{(i)}} \Delta w_{\mathrm{unit}} + \Delta w_{\min}\\
=\ &\frac{\sum_{i \in \mathcal{C}} v^{(i)} \frac{\Delta w^{(i)} - \Delta w_{\min}}{\Delta w_{\mathrm{unit}}}}{\sum_{i\in\mathcal{C}} v^{(i)}} \Delta w_{\mathrm{unit}} + \Delta w_{\min}\\
=\ &\frac{\sum_{i \in \mathcal{C}} v^{(i)}\left(\Delta w^{(i)} - \Delta w_{\min}\right)}{\sum_{i\in\mathcal{C}} v^{(i)}} + \Delta w_{\min}\\
=\ &\frac{\sum_{i \in \mathcal{C}} v^{(i)} \Delta w^{(i)}}{\sum_{i\in\mathcal{C}} v^{(i)}} - \frac{\sum_{i \in \mathcal{C}} v^{(i)} \Delta w_{\min}}{\sum_{i\in\mathcal{C}} v^{(i)}} + \Delta w_{\min}\\
=\ &\frac{\sum_{i \in \mathcal{C}} v^{(i)} \Delta w^{(i)}}{\sum_{i\in\mathcal{C}} v^{(i)}} - \Delta w_{\min} + \Delta w_{\min}\\
=\ &\frac{\sum_{i \in \mathcal{C}} v^{(i)} \Delta w^{(i)}}{\sum_{i\in\mathcal{C}} v^{(i)}},
\end{align*}
which is the same as the desired outcome in expectation.

In a nutshell, we can conclude that secure federated submodel learning does not introduce any error/bias under the $\gamma$-level stochastic quantization mechanism.

\subsection{Divergence of Conventional Federated Learning}\label{app:divergence}

The behavior of divergence has ever been observed due to large local epoch numbers in~\cite{proc:aistats17:fedavg}, which initially proposed the federated averaging algorithm. However, the local epoch number in our evaluation is set to the minimum 1, which cannot account for divergence here. As illustrated in Section~\ref{sec:evaluation:accuracy}, the major reason is that conventional federated learning inaccurately counts in the weights (\ie, the training set sizes) of those clients who contribute zero/no updates for some network parameters (\eg, some embedding vectors in DIN) when computing the weighted average of the full model updates. We examine an embedding vector for the goods with ID 1 for example. We consider that 100 clients are chosen, and assume that the sizes of their training sets are all 300s. Additionally, only 10 samples in client 1's training set involves goods 1 while the samples of the other 99 clients do not, which implies that only client 1 updates the embedding vector for goods 1 while the others do not. After each client trains locally for one epoch and uploads update, under the conventional federated learning framework, the cloud server will update the global model by adding $300/(300 \times 100) = 1\%$ of the client 1's update of the embedding vector, where the total weights count in the training set sizes of the other 99 clients who contribute zero updates. In contrast, under our federated submodel learning framework, the cloud server will add $10/10 = 100\%$ of the client 1's update of the embedding vector to the global model, where the weight of client 1 along with the total weights only count in the size of involved samples. By comparing with centralized training using the same hyperparameters over all 100 clients' data in sequence for one epoch, which will update the global model by adding $100\%$ rather than $1\%$ of the client 1's update of the embedding vector, we can find that federated averaging of the full model updates may cause some network layers, which do not involve each client's whole training set, to be trained in an inaccurate way. This further indicates that we should leverage the fine-grained involved training set sizes (\eg, at the level of individual embedding vector here) as weights like in our federated submodel learning, rather than using the whole training set size at the level of the full model as a unified weight like in conventional federated learning. Of course, the roughness and inaccuracy of conventional federated learning are completely exposed in our e-commerce context, mainly due to the high heterogeneity of user data. In particular, the full size of goods and category IDs are huge, and different Taobao users tend to have highly differentiated or even mutually exclusive sets of goods and category IDs, thus involving and updating different rows of the embedding matrix. Nevertheless, in some other contexts, where user data and their truly required submodels are not heterogenous enough, these shortcomings may not appear. For example, in the natural language context considered by Google, clients use a small vocabulary of size 10,000 for local training in Gboard, which is similar to the full set of goods and category IDs but with a much smaller scale. Therefore, for the embedding vector of a certain word, federated averaging results of its updates using the size of a client's whole text samples and the size of the samples involving this word may not differ too much.

\subsection{Leakage of A Client's Real Index Set in Multiple Communication Rounds and Countermeasures}

We now introduce a privacy leakage that if a client participates in multiple communication rounds of federated (submodel) learning even with the strongest security and privacy guarantees, the cloud server, as an adversary, may reveal a client's real indices. We focus on client $i$ with her real index set $\mathcal{S}^{(i)}$ and assume that client $i$ participates in two communication rounds, where the unions of chosen clients' real index sets are denoted as $\mathcal{U}_1$ and $\mathcal{U}_2$, respectively. We note that $\mathcal{U}_1$ and $\mathcal{U}_2$ are both revealed to the cloud server, no matter whether in our secure federated submodel learning or in secure federated learning. First, the leakages of $\mathcal{U}_1$ and $\mathcal{U}_2$ in secure federated submodel learning are trivial. Regarding secure federated learning, the cloud server can still learn $\mathcal{U}_1$ and $\mathcal{U}_2$ from two aggregate full model updates by filtering out those indices with zero updates. Please see Fig.~\ref{fig:adversary:model:full} for an intuition, and refer to the proof of Theorem~\ref{theorem:same:security} for formal reasoning. Now, the cloud server computes the intersection of $\mathcal{U}_1$ and $\mathcal{U}_2$, which must contain $\mathcal{S}^{(i)}$, \ie, $\mathcal{S}^{(i)} \subset \mathcal{U}_1 \bigcap \mathcal{U}_2$. Here, we can further derive that: (1) if client $i$ participates in more communication rounds, due to the properties of the intersection operation, the cloud server can narrow down the scope of $\mathcal{S}^{(i)}$; (2) if the real index sets of different clients are mutually exclusive, and the cloud server selects totally different sets of clients (excluding client $i$) in two communication rounds, then $\mathcal{S}^{(i)} = \mathcal{U}_1 \bigcap \mathcal{U}_2$.

To mitigate the above leakage, we give the following four kinds of countermeasures. (1) First, we adopt the concept of ``period" introduced in Section~\ref{sec:design:index:set:perturbation}, where a period comprises multiple communication rounds, and within a certain period, each client just uses his historical data in the previous one period to participate in federated (submodel) learning. To avoid the above leakage, we here enforce that each client is only chosen to join in one communication round in one period. Under such a circumstance, the cloud server only obtains $\mathcal{U}_1$ and knows $\mathcal{S}^{(i)} \subset \mathcal{U}_1$, but cannot further get $\mathcal{U}_2$ and more unions to narrow down the scope of $\mathcal{S}^{(i)}$. (2) Second, we adopt the concept of ``group", where the clients in a group participate in federated (submodel) learning together. For example, if client $i$ joins in a group of 50 clients, then the cloud server may only learn the union of these 50 clients' real index sets, even if this group participates in an infinite number of communication rounds. (3) Third is resorting to anonymization. For example, a client can use pseudo identities to participate in different communication rounds, which does not affect the execution of secure federated (submodel) learning. However, such an anonymization process breaks the linkability and thus disables the intersection operation between $\mathcal{U}_1$ and $\mathcal{U}_2$ because the cloud server does not know whether a specific client participates in both communication rounds or not. (4) Last is to let each client replace her real index set with a perturbed index set to participate in federated (submodel) learning. More specifically, just as index set perturbation in Algorithm~\ref{alg:index:set:ldp}, each client only keeps part of her real indices, randomly adds some other indices from the full index set or other clients' indices in her previously involved rounds, and thus generates a perturbed index set locally. Even though a client is chosen for an infinite number of communication rounds, the cloud server may only reveal her perturbed index set.
\end{appendix}

\section*{Acknowledgment}
We would like to sincerely thank some members of Advanced Network Laboratory (ANL) in Shanghai Jiao Tong University, including Renjie Gu, Hongtao Lv, Hejun Xiao, and Zhenzhe Zheng, as well as many colleagues in Alibaba, for meaningful discussions and great engineering support.

\bibliographystyle{IEEEtran}
\bibliography{reference}

\end{document}